\documentclass[12pt]{article}

\usepackage{amsfonts}
\usepackage{amssymb}
\usepackage{bm}
\usepackage{amsmath}
\usepackage{amsthm}
\usepackage{stmaryrd}
\usepackage{color}
\usepackage{xcolor}
\usepackage[subnum]{cases}
\usepackage{rotating}
\usepackage{enumitem}
\usepackage{accents}
\usepackage[title]{appendix}
\usepackage{mathtools}
\usepackage{caption}
\usepackage{url}
\usepackage{framed}
\usepackage{lscape}
\usepackage{multirow}
\usepackage{latexsym}
\usepackage{mathrsfs}
\usepackage{array}
\usepackage{makecell}
\usepackage{float}
\usepackage{seqsplit}
\usepackage{microtype}

\usepackage{tablefootnote}
\usepackage{tikz}
\usetikzlibrary{arrows.meta,positioning,calc}

\usepackage[usestackEOL]{stackengine}
\usepackage{pdflscape}
\usepackage{geometry}
\DeclareMathOperator{\R}{\mathbb{R}}

\DeclareMathOperator{\grad}{\nabla}

\newtheorem{lemma}{Lemma}[section]

\newtheorem{theorem}{Theorem}[section]
\newtheorem{remark}{Remark}[section]

\definecolor{lightgreen}{rgb}{0.1,0.5,0.1}
\usepackage[colorlinks=true,
breaklinks=true,
bookmarks=true,
urlcolor=blue,
citecolor=lightgreen,
linkcolor=lightgreen,
bookmarksopen=false,
draft=false]{hyperref}

\usepackage{graphicx}
\usepackage{multirow}
\usepackage{hhline}

\usepackage{appendix}

\usepackage{color}

\usepackage{amssymb}
\usepackage{mathtools}
\usepackage{stmaryrd}
\expandafter\def\expandafter\normalsize\expandafter{%
    \normalsize%
    \setlength\abovedisplayskip{5pt}%
    \setlength\belowdisplayskip{3pt}%
}

\usepackage{multirow}
\usepackage{array}
\usepackage{makecell}
\usepackage{float}
\usepackage{tablefootnote}
\usepackage{tikz}

\usepackage{enumitem}

\usepackage{url}

\usepackage[numbers]{natbib}
\usepackage[T1]{fontenc}
\usepackage[utf8]{inputenc}
\usepackage{amsmath}
\usepackage{amssymb}
\usepackage{graphicx}
\usepackage{dsfont}
\usepackage{amsmath}
\usepackage{array}
\usepackage{multirow}
\usepackage{caption}
\usepackage{color}

\usepackage{multicol}

\makeatother
  
\usepackage{babel} 

\providecommand{\claimname}{Claim}
\providecommand{\lemmaname}{Lemma}
\providecommand{\propositionname}{Proposition}
\providecommand{\theoremname}{Theorem}
\providecommand{\corollaryname}{Corollary} 
\providecommand{\definitionname}{Definition}
\providecommand{\assumptionname}{Assumption}
\providecommand{\remarkname}{Remark}

\newcommand{\diag}{\mathrm{diag}}

\newcommand{\Tr}{\mathrm{Tr}}

\newcommand{\RR}{\mathbb{R}}

\usepackage{enumitem}

\begin{document}

\title{Implicit Bias via Spectral Neural Networks \protect\\ and Non-Linear Matrix Sensing}

\author{Hong T.M. Chu\footnotemark[1] \;\footnotemark[2]\; \quad Xiaoyu Dong\footnotemark[1]\; \footnotemark[2] \;\footnotemark[5] \quad Subhro Ghosh\footnotemark[1]\; \footnotemark[2] \;\footnotemark[5] \protect \\  Chi Thanh Lam\footnotemark[1]\; \footnotemark[3]\;  \quad Soumendu Sundar Mukherjee\footnotemark[1]\; \footnotemark[4]\; \footnotemark[5]}

\maketitle

\renewcommand{\thefootnote}{\fnsymbol{footnote}}
\footnotetext[1]{Authors are listed in the alphabetical order of their surnames}
\footnotetext[2]{Dept of Mathematics, National University of Singapore}
\footnotetext[3]{School of Computing, National University of Singapore}
\footnotetext[4]{Statistics and Mathematics Unit, Indian Statistical Institute}
\footnotetext[5]{X.D., S.G., and S.S.M. are corresponding authors.}
\renewcommand{\thefootnote}{\arabic{footnote}}

\begin{abstract}
	The phenomenon of \textit{implicit bias} has attracted interest in recent years as a fundamental aspect of the remarkable generalizing ability of neural networks. In a nutshell, it entails that gradient descent dynamics in many neural nets, even without any explicit regularizer in the loss function, converges to the solution of a regularized learning problem.  
    In this work, we show that this phenomenon extends to non-linear
spectral parameterizations in matrix sensing. More precisely, we propose a matrix network architecture called Spectral Neural Networks (\textit{abbrv.} SNN) where each neuron applies nonlinear activation to the singular values of the input matrix. Conceptually, this entails coordinatizing the space of matrices by their singular values and singular vectors, as opposed to by their entries, a potentially fruitful perspective for matrix learning. We prove the exponentially fast decay of squared loss during the training dynamics of an SNN parameterized matrix sensing model under gradient flow (which is a continuous approximation of the discrete gradient descent), and also that the training implicitly  selects an approximate solution that minimizes nuclear norm. We also demonstrate the effectiveness of the SNN architecture in the context of matrix sensing by empirical investigations. We believe that the SNN architecture has the potential to be of wide applicability in a broad class of matrix learning scenarios.
\end{abstract}

\medskip

\section{Introduction}

A longstanding pursuit of deep learning theory is to explain the astonishing ability of neural networks to generalize despite having far more learnable parameters than training data, even in the absence of any explicit regularization. An established understanding of this phenomenon is that the gradient descent algorithm induces a so-called \emph{implicit bias} effect. In very general terms, implicit bias entails that gradient flow dynamics in many neural nets, even without any explicit regularizer in the loss function, converges to the solution of a regularized learning problem. In a sense, this creates a learning paradigm that automatically favors models characterized by ``low complexity''.

In an attempt to explain the generalizing ability of
over-parameterized neural networks, \cite{neyshabur2014search}
 suggested the idea of an implicit bias effect of the
optimizer, which entails a bias towards solutions that generalize
well (see also \cite{neyshabur2017exploring}).

A natural starting point to understand implicit bias is to analyze it in linear models
\citep{soudry2018implicit, zhao2019implicit} and linear neural
networks \citep{gidel2019implicit}. A common approach to understanding implicit bias is through implicit norm regularization, where the optimization dynamics favor solutions with small norms \citep{zhao2019implicit}. 
Besides norm-regularization, several of these works demonstrate
the implicit bias effect of gradient descent in terms of other
relevant quantities such as margin \citep{soudry2018implicit}
or structural bias \citep{gidel2019implicit}. Closely related works analyze training dynamics in linear neural
networks \citep{saxe2019mathematical}. In the over-parameterized
setting, \cite{du2019width} studies the role of width in the
convergence of gradient descent. Focusing on gradient flow,
\cite{tarmoun2021understanding} studies the conservation of the
difference between the Gram matrices of the factors and its role
in the training dynamics. In related work,
\cite{min2021explicit} studies implicit bias towards the minimum Frobenius norm
solution in linear regression under suitable conditions.
For multi-layer linear models, \cite{min2023convergence} studies
exponential convergence using gradient dominance and
initialization-dependent spectral bounds.

In the non-linear setting, implicit
bias in neural networks was studied in
\citep{li2018algorithmic, blanc2020implicit, kubo2019implicit}.
Several of these works demonstrate the implicit bias effect of
gradient-based methods in terms of other relevant quantities
such as the number of times the model changes its convexity
under stochastic gradient descent with label noise
\citep{blanc2020implicit}, or linear interpolation
\citep{kubo2019implicit}.
In this setting, \cite{jinmontufar2023implicit} characterizes
the implicit bias of gradient descent for mean squared error
regression with wide two-layer ReLU networks.

From a spectral perspective, \cite{bowmanmontufar2022mse} studies
``spectral bias'' under squared-loss gradient flow in the
underparameterized setting, while
\cite{bowmanmontufar2022spectral} studies spectral bias outside
the training set for deep networks in the kernel regime.
The spectral bias in these latter works concerns the neural
tangent kernel.

A standard test-bed and natural use case scenario  for investigating the implicit bias
phenomenon is the problem of \emph{matrix sensing}. 
The goal is to approximate a matrix $X^\star$ from a set of measurement matrices $A_1,\dots,A_m$ and observations $y_1,\dots,y_m$ where $y_i=\langle A_i, X^\star \rangle$. A common approach, \emph{matrix factorization}, parameterizes the solution as a product matrix, i.e., $X=UV^\top$, and optimizes the resulting non-convex objective to fit the data. This is equivalent to training a depth-$2$ neural network with a linear activation function.

Classical works in
matrix sensing and matrix factorization study rank and norm
constraints \citep{srebro2005rank}, utilize convex relaxation
approaches, i.e., minimizing the nuclear norm subject to agreement
with the observations, and derive near optimal sample complexity bounds
\citep{candes2009exact, recht2010guaranteed, candes2010power,
recht2011simpler}. For non-convex approaches, several works have analyzed gradient
descent on matrix sensing and matrix completion
\citep{keshavan2010matrix}. In particular, \cite{ge2016matrix}
and \cite{bhojanapalli2016global} show that the non-convex
objectives on matrix sensing and matrix completion with low-rank
parameterization do not have any spurious local minima.
Their analyses also establish convergence to the global minimum
for suitable gradient-based algorithms. Beyond convergence guarantees, \cite{gunasekar2017implicit}
investigated the possibility of an implicit norm-regularization
effect of gradient descent in the context of shallow matrix
factorization. In particular, they studied the standard
Burer-Monteiro approach \cite{burer2003nonlinear} to matrix
factorization, which may be viewed as a depth-2 linear neural
network. They were able to theoretically demonstrate an implicit
norm-regularization phenomenon, where a suitably initialized
gradient flow dynamics approaches a solution to the nuclear-norm
minimization approach to low-rank matrix recovery
\cite{recht2010guaranteed}, in a setting where the involved
measurement matrices commute with each other. They also
conjectured that this latter restriction on the measurement
matrices is unnecessary. This conjecture was later resolved by
\cite{li2018algorithmic} in the setting where the measurement
matrices satisfy a restricted isometry property.

The limitations of norm-based descriptions of implicit bias in matrix
factorization problems were investigated in several follow-up
papers \citep{arora2019implicit, razin2020implicit}. For instance,
\cite{arora2019implicit} questioned whether the implicit
norm-regularization property of gradient flow, as studied by
\cite{gunasekar2017implicit}, holds in general in the context of
deep matrix factorization. More concretely,
\cite{razin2020implicit} constructed a simple $2 \times 2$ example,
where the gradient flow dynamics lead to an eventual blow-up of
any matrix norm, while a certain relaxation of rank (the so-called
\emph{e-rank}) is minimized in the limit. These works suggest
that implicit bias in deep networks should be interpreted through
the lens of rank minimization, not norm minimization alone.
In the context of tensor factorization, \cite{razin2021implicit}
have demonstrated a related bias towards low tensor rank.

In this paper, we investigate the implicit bias phenomenon
for non-linear spectral parameterizations in matrix sensing.
We show that under our reparameterization, an appropriately
initialized gradient flow dynamics leads to the minimization
of the nuclear norm, which is a convex proxy of rank minimization.

In the bigger picture, we contribute a network architecture that we refer to as the Spectral Neural Network architecture (abbrv. SNN), which is particularly suitable for matrix learning scenarios. Conceptually, this entails coordinatizing the space of matrices by their singular values and singular vectors, as opposed to by their entries. We believe that this point of view can be beneficial for tackling matrix learning problems in a neural network setup. SNNs are particularly well-suited for theoretical analysis due to the spectral nature of their non-linearities, as opposed to vanilla neural nets, while at the same time (provably) guaranteeing effectiveness in matrix learning problems.  We also introduce a much more general counterpart of the near-zero initialization that is popular in related literature, and our methods are able to handle a much more robust class of initializing setups that are constrained only via certain inequalities. Our theoretical contributions include a compact analytical representation of the gradient flow dynamics, accorded by the spectral nature of our network architecture. We demonstrate the efficacy of the SNN architecture through its application to the matrix sensing problem, bolstered via both theoretical guarantees and extensive empirical studies. We believe that the SNN architecture has the potential to be of wide applicability in a broad class of matrix learning problems. In particular, we believe that the SNN architecture would be natural for the study of rank (or e-rank) minimization effect of implicit bias in deep matrix/tensor factorization problems, especially given the fact that e-rank is essentially a spectrally defined concept.

\section{Matrix Sensing by Spectral Neural Networks}
\subsection{Matrix sensing}
We first introduce some notation. For $X\in\mathbb R^{d_1\times d_2}$, let $\sigma_j(X)$ denote
its singular values. Recall that the operator norm and the nuclear norm are
defined by
\[
 \|X\|_{\mathrm{op}}
 :=\max_{1\le j\le\min\{d_1,d_2\}}\sigma_j(X),
 \qquad
 \|X\|_*
 :=\sum_{j=1}^{\min\{d_1,d_2\}}\sigma_j(X),
\]
and the Frobenius norm is
\[
    \|X\|_F:=\sqrt{\langle X,X\rangle}.
\]
Also, for a vector $v=(v_j)_{j=1}^n$, we write
\[
 \|v\|_1:=\sum_{j=1}^n|v_j|,
 \qquad
 \|v\|_2:=\left(\sum_{j=1}^n v_j^2\right)^{1/2},
 \qquad
 \|v\|_\infty:=\max_{1\le j\le n}|v_j|.
\]
Let $X^{\star} \in \RR^{d_1 \times d_2}$ be an unknown rectangular matrix that we aim to recover. 
Let $A_1, \dots, A_m \in \RR^{d_1 \times d_2}$ be $m$ measurement matrices, and let the label vector $y \in \RR^m$ be generated by `sensing' $X^{\star}$ as follows:
\begin{equation}
    y_i = \langle A_i, X^{\star} \rangle,
\end{equation}
where $\langle A, B \rangle = \Tr(A^\top B)$ denotes the Frobenius inner product.  We consider the following squared loss objective
\begin{equation} \label{eq:loss}
    \ell(X) := \frac1{2m} \sum_{i=1}^m (y_i - \langle A_i, X \rangle)^2.
\end{equation}
For convenience, let \begin{equation}
\mathcal{A}(D)
:=
\frac{1}{\sqrt{m}}
\begin{bmatrix}
\langle A_1, D\rangle \\
\langle A_2, D\rangle \\
\vdots \\
\langle A_m, D\rangle
\end{bmatrix}
\end{equation}
so that we can write
\begin{equation}
 \ell(X)=\frac1{2m}\sum_{i=1}^m
       (y_i-\langle A_i,X\rangle)^2
       =\frac12\|\mathcal{A}(X-X^\star)\|_2^2.
\end{equation}
In other words, from the perspective of machine learning theory, $y_1,...,y_m$ are the labels generated as per the \emph{noiseless regression model}
\begin{equation}
    y_i = \langle A_i, X^{\star} \rangle=\Tr(A_i^\top X^{\star}),
\end{equation}
where $A_1, \dots, A_m \in \RR^{d_1 \times d_2}$ are features and $X^{\star} \in \RR^{d_1 \times d_2}$ is an unknown true rectangular parameter matrix that we aim to learn. Since our focus is on the training behavior, for simplicity, the feature data $A_1, \dots, A_m$ are considered to be deterministic. This is without loss of generality because they can also be considered as realizations of random matrices sampled i.i.d. from some probability distribution.

This setting covers problems including matrix completion (where the $A_i$'s are indicator matrices), matrix sensing from linear measurements, and multi-task learning (in which the columns of $X$ are predictors for the tasks, and each $A_i$ has only one non-zero column). 
We are interested in the overparameterized regime where $m \ll d_1 \times d_2$, i.e., the number of measurements is much less than the number of entries in $X^{\star}$, in which case optimization of \eqref{eq:loss} leads to an underdetermined system with many global minima. Therefore, merely an attempt at minimizing \eqref{eq:loss} does not guarantee correct recovery or good generalization.

\subsection{Reparameterization by spectral neural network}\label{sec:SNNreparameterization}
In this work, we propose the following (neural network type) reparameterization of $X$ as
\begin{equation} \label{eq:X}
    X =X_{\rm train}(\alpha,U,V):= \sum_{k=1}^{K} \alpha_k \Gamma(U_k V_k^\top),
\end{equation}
where $\alpha \in \RR$, $d \ge \min\{d_1,d_2\}$, $U_k \in \RR^{d_1 \times d}, V_k \in \RR^{d_2 \times d}$, $K$ is a positive integer representing the width of this neural network structure, and the matrix-valued function $\Gamma:\RR^{d_1 \times d_2} \rightarrow \RR^{d_1 \times d_2}$ transforms a matrix by applying a nonlinear real-valued function $\gamma(\cdot)$ on its singular values. More precisely, for $W\in\R^{d_1\times d_2}$,
consider its reduced singular value decomposition
\[
 W=\Phi\, \diag(w_1,\ldots,w_{\min\{d_1, d_2\}})\Psi^\top,
\]
where $\Phi\in\R^{d_1\times \min\{d_1, d_2\}}$ and $\Psi\in\R^{d_2\times \min\{d_1, d_2\}}$
have orthonormal columns and $w_j\ge0$.
For an odd real-analytic scalar activation $\gamma$, define
the \emph{spectral activation} by
\[
 \Gamma(W)
 :=
 \Phi\,\diag\bigl(\gamma(w_1),\ldots,\gamma(w_{\min\{d_1, d_2\}})\bigr)\Psi^\top
 \in\R^{d_1\times d_2}.
\]
We call the resulting architecture the spectral neural network (SNN), with each $\Gamma(U_k V_k^\top)$ representing a `spectral neuron' (see section \ref{sec:snn-architecture} for details). Here we use SNNs not as a hypothesis class modeling the relationship between features $A_i$ and labels $y_i$, but as a reparameterization that generates the unknown parameter matrix $X^\star$.

For our analysis, we assume that the nonlinear activation $\gamma$ satisfies
\begin{equation}\label{eq:activation-family}
 \begin{aligned}
 &(i) \,\,\,\quad \gamma \text{ is odd and real analytic},\\
 &(ii) \,\quad \gamma'(0)=1,\\
 &(iii) \quad \text{for any } z\in\mathbb R, \text{ we have } \gamma'(z)>0.
 \end{aligned}
\end{equation}

\subsection{Training dynamics and gradient flow}

For convenience, here we perform calculations using general differentiation rules between normed vector spaces (see, e.g., \cite{zorich2016mathematical}).  Recall that the (Fr\'echet) differential $Df(x)$ of a function $f$ between normed spaces is the continuous
linear map satisfying
\begin{equation}\label{eq:differential}
 f(x+h)-f(x)=Df(x)[h]+o(h).
\end{equation}
On a product space, $D_{x_i}f(x)$ denotes the differential of the partial
mapping obtained by fixing every block except $x_i$. With these conventions, for
$h=(h_1, h_2, \ldots)$, the total differential and the chain rule are given respectively by
\[
 \begin{aligned}
 Df(x)[h]&=\sum_i D_{x_i}f(x)[h_i], \quad \text{and}\\
 D(g\circ f)(x)[h]&=Dg(f(x))[Df(x)[h]].
 \end{aligned}
\]
Note that for differentiable $f$, the chain rule also gives
$Df(x)[h]=\left.\frac d{dt}f(x+th)\right|_{t=0}$.
In addition, the second differential is written $D^2f(x)[h_1,h_2]$.
For a scalar function of one real variable we recover the scalar differential
$Dg(z)[a]=g'(z)a$.

For convenience, we write $L(\alpha,U,V):=\ell(X_{\rm train}(\alpha,U,V))$. Also, recall that the gradient
is the representing vector of the differential, i.e.,
$Df(x)[h]=\langle\grad f(x),h\rangle$. In particular, the partial differential $D_{U_k}L(\alpha,U,V)$ is a
linear functional, and its representing matrix is defined by
\begin{equation}\label{eq:matrix-gradient}
 D_{U_k}L(\alpha,U,V)[P]
 =\langle\grad_{U_k}L(\alpha,U,V),P\rangle
 \qquad\text{for every }P\in\R^{d_1\times d}.
\end{equation}

For each $k \in [K]$, let $\alpha_k(t), U_k(t), V_k(t)$ denote the trajectories of gradient flow, where $\alpha_k(0), U_k(0), V_k(0)$ are the initial conditions. Therefore, we have \[X_{\rm train}(\alpha(t),U(t),V(t))=\sum_{k=1}^{K} \alpha_k(t) \Gamma(U_k(t) V_k(t)^\top).\] The dynamics of gradient flow are given by the following differential equations: for $k \in [K]$,
\begin{align}\label{eq:gradient-flow}
 \dot\alpha_k:&=\frac{d}{dt}\alpha_k =-\partial_{\alpha_k}L(\alpha,U,V),
 \\
 \dot U_k:&=\frac{d}{dt}U_k =-\grad_{U_k}L(\alpha,U,V),\nonumber
 \\
 \dot V_k:&=\frac{d}{dt}V_k=-\grad_{V_k}L(\alpha,U,V).\nonumber
\end{align}
Along the trajectories $\alpha_k(t), U_k(t), V_k(t)$ of the gradient flow, $L(\alpha(t),U(t),V(t))$ is a function of $t$.
In this case, for convenience, when the focus is to analyze how the loss in the gradient flow evolves with $t$, we use the shorthand notation
\[
    L(t):=L(\alpha(t),U(t),V(t)).
\]

\begin{lemma}
The gradient flow satisfies \begin{equation}\label{eq:ambient-flow}
 \begin{aligned}
 \dot\alpha_k&=-\langle\Gamma(U_kV_k^\top),\mathcal A^*\mathcal A(X-X^\star)\rangle,\\
 \dot U_k&=-\alpha_k\bigl(D\Gamma(U_kV_k^\top)[\mathcal A^*\mathcal A(X-X^\star)]\bigr)V_k,\\
 \dot V_k&=-\alpha_k\bigl(D\Gamma(U_kV_k^\top)[\mathcal A^*\mathcal A(X-X^\star)]\bigr)^\top U_k.
 \end{aligned}
\end{equation}
\end{lemma}
\begin{proof}
For every increment
$E\in\mathbb R^{d_1\times d_2}$, we have
\[
 D\ell(X)[E]
 =\langle\mathcal A(X-X^\star),\mathcal A E\rangle
 =\langle\mathcal A^*\mathcal A(X-X^\star),E\rangle.
\]
Therefore, we conclude \[\grad_X\ell(X)=\mathcal A^*(\mathcal A(X-X^\star))\]
where $\mathcal A^*(\cdot)$ is the adjoint of $\mathcal A$, i.e., for $q=(q_1,\ldots,q_m)^\top$, we have
\begin{equation}
\mathcal{A}^*(q)
=
\frac{1}{\sqrt{m}}
\left(q_1 A_1+\cdots+q_m A_m\right)
\in \mathbb{R}^{d_1\times d_2}.
\end{equation}

By chain rule, we calculate
 \[
 \begin{aligned}
 D_{\alpha_k}X(\alpha,U,V)[a]&=a\Gamma(U_kV_k^\top),\\
 D_{U_k}X(\alpha,U,V)[P]&=\alpha_k D\Gamma(U_kV_k^\top)[PV_k^\top],\\
 D_{V_k}X(\alpha,U,V)[Q]&=\alpha_k D\Gamma(U_kV_k^\top)[U_kQ^\top].
 \end{aligned}
\]

With these preparations, we can calculate the gradients. By chain rule, e.g.,
\[
 D_{U_k}L(\alpha,U,V)[P]
 =D\ell(X_{\mathrm{train}})
   \bigl[D_{U_k}X_{\mathrm{train}}(\alpha,U,V)[P]\bigr],
\]
we obtain
\begin{equation}\label{eq:factor-chain-rule}
 \begin{aligned}
 D_{\alpha_k}L(\alpha,U,V)[a]
 &=a\langle\Gamma(U_kV_k^\top),\mathcal A^*\mathcal A(X_{\mathrm{train}}-X^\star)\rangle,\\
 D_{U_k}L(\alpha,U,V)[P]
 &=\alpha_k\langle\mathcal A^*\mathcal A(X_{\mathrm{train}}-X^\star),D\Gamma(U_kV_k^\top)[PV_k^\top]\rangle\\
 &=\left\langle
 \alpha_k\bigl((D\Gamma(U_kV_k^\top))^*[\mathcal A^*\mathcal A(X_{\mathrm{train}}-X^\star)]\bigr)V_k,P
 \right\rangle,\\
 D_{V_k}L(\alpha,U,V)[Q]
 &=\alpha_k\langle\mathcal A^*\mathcal A(X_{\mathrm{train}}-X^\star),D\Gamma(U_kV_k^\top)[U_kQ^\top]\rangle\\
 &=\left\langle
 \alpha_k\bigl((D\Gamma(U_kV_k^\top))^*[\mathcal A^*\mathcal A(X_{\mathrm{train}}-X^\star)]\bigr)^\top U_k,Q
 \right\rangle.
 \end{aligned}
\end{equation}
By standard calculations, we can verify that $D\Gamma(U_kV_k^\top)$ is self-adjoint. So we can replace $D\Gamma(U_kV_k^\top)^*$ by $D\Gamma(U_kV_k^\top)$ and obtain
\begin{equation*}
 \begin{aligned}
 \partial_{\alpha_k}L(\alpha,U,V)&=\langle\Gamma(U_kV_k^\top),\mathcal A^*\mathcal A(X_{\mathrm{train}}-X^\star)\rangle,\\
\grad_{U_k}L(\alpha,U,V)&=\alpha_k\bigl(D\Gamma(U_kV_k^\top)[\mathcal A^*\mathcal A(X_{\mathrm{train}}-X^\star)]\bigr)V_k,\\
\grad_{V_k}L(\alpha,U,V)&=\alpha_k\bigl(D\Gamma(U_kV_k^\top)[\mathcal A^*\mathcal A(X_{\mathrm{train}}-X^\star)]\bigr)^\top U_k.
 \end{aligned}
\end{equation*}
In conclusion, the gradient flow satisfies
\begin{equation}
 \begin{aligned}
 \dot\alpha_k&=-\langle\Gamma(U_kV_k^\top),\mathcal A^*\mathcal A(X_{\mathrm{train}}-X^\star)\rangle,\\
 \dot U_k&=-\alpha_k\bigl(D\Gamma(U_kV_k^\top)[\mathcal A^*\mathcal A(X_{\mathrm{train}}-X^\star)]\bigr)V_k,\\
 \dot V_k&=-\alpha_k\bigl(D\Gamma(U_kV_k^\top)[\mathcal A^*\mathcal A(X_{\mathrm{train}}-X^\star)]\bigr)^\top U_k.
 \end{aligned}
\end{equation}
This completes the proof.
\end{proof}

\section{Analysis for Aligned Measurements}

For convenience, we assume that all the measurements are aligned. 
More precisely, assume that there are matrices
	$\Phi\in\mathbb R^{d_1 \times {\min\{d_1, d_2\}}}$ and $\Psi\in\mathbb R^{d_2 \times {\min\{d_1, d_2\}}}$ with
	orthonormal columns such that
	\begin{equation}\label{eq:aligned-measurements}
		A_i=\Phi\operatorname{diag}(b_i)\Psi^\top,\qquad b_i\in\mathbb R^{\min\{d_1, d_2\}}.
	\end{equation}
	The vectors $b_i$ are signed measurement coordinates, not necessarily
	singular values. Define
	$x^\star$ to be the vector containing the diagonal entries of $\Phi^\top X^\star\Psi$, even if $\Phi^\top X^\star\Psi$ is not a diagonal matrix. Then we have $y_i=b_i^\top x^\star$.
	
Let $d\ge \min\{d_1,d_2\}$.	For each $k$, choose a fixed
	$R_k\in\mathbb R^{d \times {\min\{d_1, d_2\}}}$ with $R_k^\top R_k=I_{\min\{d_1, d_2\}}$ 
    and initialize
	\begin{equation}\label{eq:aligned-init}
		U_k(0)=\Phi\operatorname{diag}(u_k(0))R_k^\top,\qquad
		V_k(0)=\Psi\operatorname{diag}(v_k(0))R_k^\top.
	\end{equation}
 For vectors of the same dimension, $a\odot b$
	denotes the Hadamard product, $(a\odot b)_j=a_jb_j$, and
	$a^{\odot2}=a\odot a$. Then we can derive the exact training dynamics in the aligned case.

\begin{theorem}\label{prop:aligned}
Under \eqref{eq:aligned-measurements} and \eqref{eq:aligned-init},
we have
$U_k(t)=\Phi\operatorname{diag}(u_k(t))R_k^\top$ and
$V_k(t)=\Psi\operatorname{diag}(v_k(t))R_k^\top$ for some $u_k(t)$ and $v_k(t)$ during training. Define $z_k(t),x(t),\xi(t)$ as
\begin{equation}\label{eq:coordinates}
 \begin{aligned}
 z_k(t)&=u_k(t)\odot v_k(t),\qquad
 x(t)=\sum_{k=1}^K\alpha_k(t)\gamma(z_k(t)),\\
 \xi(t)&=\frac1m\sum_{i=1}^m b_i\bigl(y_i-b_i^\top x(t)\bigr).
 \end{aligned}
\end{equation}
Then $X_{\rm train}(\alpha(t),U(t),V(t))=\Phi\operatorname{diag}(x(t))\Psi^\top$, and
\begin{equation}\label{eq:aligned-flow}
\begin{aligned}
\dot\alpha_k(t)&=\langle\gamma(z_k(t)),\xi(t)\rangle,\\
\dot u_k(t)&=\alpha_k(t)\gamma'(z_k(t))\odot v_k(t)\odot\xi(t),\\
\dot v_k(t)&=\alpha_k(t)\gamma'(z_k(t))\odot u_k(t)\odot\xi(t).
\end{aligned}
\end{equation}
In addition, we have
\begin{equation}\label{eq:coordinate-balance}
 \begin{aligned}
 \dot z_k(t)=\alpha_k(t)\gamma'(z_k(t))\odot
       \bigl(u_k(t)^{\odot2}+v_k(t)^{\odot2}\bigr)\odot\xi(t), \quad
 \frac d{dt}\bigl(u_k(t)^{\odot2}-v_k(t)^{\odot2}\bigr)=0.
 \end{aligned}
\end{equation}
\end{theorem}
\begin{proof}Take a local solution of \eqref{eq:aligned-flow}  with the
initialization (which exists by e.g., \cite[Theorem~2.2]{teschl2012}) and define
\[
U_k(t)=\Phi\operatorname{diag}(u_k(t))R_k^\top,
\qquad
V_k(t)=\Psi\operatorname{diag}(v_k(t))R_k^\top.
\] We want to verify these matrices, together with
$\alpha_k(t)$, satisfy the full gradient-flow ODE
\eqref{eq:ambient-flow}.

Since $R_k^\top R_k=I_{\min\{d_1, d_2\}}$, we have
\[
\begin{aligned}
U_k(t)V_k(t)^\top
&=\Phi\operatorname{diag}(u_k(t))R_k^\top R_k\operatorname{diag}(v_k(t))\Psi^\top\\
&=\Phi\operatorname{diag}(u_k(t)\odot v_k(t))\Psi^\top
=\Phi\operatorname{diag}(z_k(t))\Psi^\top.
\end{aligned}
\]
By oddness of $\gamma$, for every $w\in\mathbb R^{\min\{d_1, d_2\}}$, we also have
\[
\Gamma\bigl(\Phi\operatorname{diag}(w)\Psi^\top\bigr)
=\Phi\operatorname{diag}(\gamma(w))\Psi^\top.
\]

Summing over the $k$, we have $X_{\rm train}(t)=\Phi\operatorname{diag}(x(t))\Psi^\top$.
Since $\Phi,\Psi$ have orthonormal columns, we have
$\langle A_i,X_{\rm train}(t)\rangle=b_i^\top x(t)$. By definition of $\mathcal A$ and $\mathcal A^*$, we calculate
\[
\begin{aligned}
\mathcal A^*\mathcal A(X_{\rm train}(t)-X^\star)
&=\frac1m\sum_i(b_i^\top x(t)-y_i)\Phi\operatorname{diag}(b_i)\Psi^\top\\
&=-\Phi\operatorname{diag}(\xi(t))\Psi^\top.
\end{aligned}
\]

Next we evaluate by chain rule that
\begin{equation}\label{eq:aligned-spectral-differential}
\begin{aligned}
 &D\Gamma(U_k(t)V_k(t)^\top)[\Phi\operatorname{diag}(h)\Psi^\top]\\
 &=D\bigl(w\mapsto\Gamma(\Phi\operatorname{diag}(w)\Psi^\top)\bigr)(z_k(t))[h]\\
 &=D\bigl(w\mapsto\Phi\operatorname{diag}(\gamma(w))\Psi^\top\bigr)(z_k(t))[h]\\
 &=\Phi\operatorname{diag}(\gamma'(z_k(t))\odot h)\Psi^\top
\end{aligned}
\end{equation}
for any $h\in\mathbb R^{\min\{d_1, d_2\}}$.

Let $P\in\mathbb R^{d_1\times d}$ and $Q\in\mathbb R^{d_2\times d}$ be arbitrary
matrix increments. By chain rule and the self-adjointness of $D\Gamma(U_k(t)V_k(t)^\top)$, we have
\[
\begin{aligned}
D_{U_k}L(\alpha(t),U(t),V(t))[P]
&=D\ell(X_{\rm train}(t))\bigl[\alpha_k(t)D\Gamma(U_k(t)V_k(t)^\top)[PV_k(t)^\top]\bigr]\\
&=\alpha_k(t)\left\langle\mathcal A^*\mathcal A(X_{\rm train}(t)-X^\star),
D\Gamma(U_k(t)V_k(t)^\top)[PV_k(t)^\top]\right\rangle\\
&=-\alpha_k(t)\left\langle
D\Gamma(U_k(t)V_k(t)^\top)[\Phi\operatorname{diag}(\xi(t))\Psi^\top],PV_k(t)^\top
\right\rangle\\
&=-\alpha_k(t)\left\langle
\Phi\operatorname{diag}(\gamma'(z_k(t))\odot\xi(t))\Psi^\top,PV_k(t)^\top
\right\rangle\\
&=\left\langle
-\alpha_k(t)\Phi\operatorname{diag}(\gamma'(z_k(t))\odot v_k(t)\odot\xi(t))R_k^\top,P
\right\rangle.
\end{aligned}
\]

        Similarly, for any increment $Q$, we have
\[
\begin{aligned}
D_{V_k}L(\alpha(t),U(t),V(t))[Q]
&=D\ell(X_{\rm train}(t))\bigl[\alpha_k(t)D\Gamma(U_k(t)V_k(t)^\top)[U_k(t)Q^\top]\bigr]\\
&=-\alpha_k(t)\left\langle
D\Gamma(U_k(t)V_k(t)^\top)[\Phi\operatorname{diag}(\xi(t))\Psi^\top],U_k(t)Q^\top
\right\rangle\\
&=-\alpha_k(t)\left\langle
\Phi\operatorname{diag}(\gamma'(z_k(t))\odot\xi(t))\Psi^\top,U_k(t)Q^\top
\right\rangle\\
&=\left\langle
-\alpha_k(t)\Psi\operatorname{diag}(\gamma'(z_k(t))\odot u_k(t)\odot\xi(t))R_k^\top,Q
\right\rangle
\end{aligned}
\]
and for any increment $a$, we have
\[
\begin{aligned}
D_{\alpha_k}L(\alpha(t),U(t),V(t))[a]
&=D\ell(X_{\rm train}(t))[a\Gamma(U_k(t)V_k(t)^\top)]\\
&=a\langle\Gamma(U_k(t)V_k(t)^\top),\mathcal A^*\mathcal A(X_{\rm train}(t)-X^\star)\rangle\\
&=-a\langle\gamma(z_k(t)),\xi(t)\rangle.
\end{aligned}
\]
        These calculations give the partial gradients, which verifies the full gradient-flow ODE
\eqref{eq:ambient-flow}.
Also, this solution is unique by \cite[Theorem~2.2, Theorem~2.13]{teschl2012}.

It remains to prove \eqref{eq:coordinate-balance}.
The Hadamard product is bilinear. For arbitrary increments  $h,q\in\mathbb R^{\min\{d_1, d_2\}}$, we have
\[
D\bigl((u,v)\mapsto u\odot v\bigr)(u_k(t),v_k(t))[h,q]
=h\odot v_k(t)+u_k(t)\odot q,
\]
because the remaining term is $h\odot q$, whose norm is at most
$\|h\|_2\|q\|_2$. Applying the same rule to $w\odot w$, we obtain
$D(w\mapsto w^{\odot2})(u_k(t))[h]=2u_k(t)\odot h$.
By chain rule and \eqref{eq:aligned-flow}, we have
\[
\begin{aligned}
\dot z_k(t)
&=\dot u_k(t)\odot v_k(t)+u_k(t)\odot\dot v_k(t)\\
&=\alpha_k(t)\gamma'(z_k(t))\odot
(u_k(t)^{\odot2}+v_k(t)^{\odot2})\odot\xi(t),\\
\frac d{dt}(u_k(t)^{\odot2}-v_k(t)^{\odot2})
&=2u_k(t)\odot\dot u_k(t)-2v_k(t)\odot\dot v_k(t)\\
&=2\alpha_k(t)\gamma'(z_k(t))\odot\xi(t)\odot
(u_k(t)\odot v_k(t)-v_k(t)\odot u_k(t))=0.
\end{aligned}
\]
\end{proof}

\begin{lemma}\label{prop:contraction}
Under \eqref{eq:aligned-measurements} and \eqref{eq:aligned-init}, we have
\begin{equation}\label{eq:output-flow}
 \begin{aligned}
 \dot{x}(t)&=M(t)\xi(t), \quad \text{where } \\
M(t)&=\sum_{k=1}^K\gamma(z_k(t))\gamma(z_k(t))^\top\\
 &\quad+\operatorname{diag}\!\left(\sum_{k=1}^K
 \alpha_k(t)^2(\gamma'(z_k(t)))^{\odot2}
 \odot(u_k(t)^{\odot2}+v_k(t)^{\odot2})\right)\succeq 0.
 \end{aligned}
\end{equation}
Let $\lambda_+>0$ be the smallest positive eigenvalue of the nonzero
matrix $m^{-1}\sum_i b_i b_i^\top$.
If $M(t)\succeq\mu I_{\min\{d_1, d_2\}}$ on an arbitrary interval $[t_0,T]$, with $\mu>0$, then
\begin{equation}\label{eq:coordinate-contraction}
 L(t)\le L(t_0)e^{-2\mu\lambda_+(t-t_0)}
 \qquad(t_0\le t\le T),
\end{equation}
where we recall the shorthand notation $L(t)=L(\alpha(t),U(t),V(t))$.
\end{lemma}
\begin{proof}
The product and chain rules give
\[
 \dot{x}(t)=\sum_k\left(
 \dot{\alpha}_k(t)\gamma(z_k(t))
 +\alpha_k(t)\gamma'(z_k(t))\odot\dot{z}_k(t)\right).
\]
Substituting \eqref{eq:aligned-flow}--\eqref{eq:coordinate-balance}
proves \eqref{eq:output-flow}. Therefore, we have \begin{equation}\label{eq:dissipation}
 -\frac d{dt}L(t)
 =\xi(t)^\top M(t)\xi(t)
 =\bigl\|(\dot{\alpha}(t),\dot{U}(t),\dot{V}(t))\bigr\|^2.
\end{equation}

Since $y_i=b_i^\top x^\star$, by the definition of $\xi(t)$, we have
\[
\begin{aligned}
\xi(t)
&=\frac1m\sum_{i=1}^m b_i
  \bigl(b_i^\top x^\star-b_i^\top x(t)\bigr)\\
&=\left(\frac1m\sum_{i=1}^m b_i b_i^\top\right)
  \bigl(x^\star-x(t)\bigr)
  \end{aligned}
\]
For every eigenvalue $\lambda\ge0$ of $m^{-1}\sum_i b_i b_i^\top$, we have
$\lambda^2\ge\lambda_+\lambda$.  So we have
\begin{equation}\label{eq:aligned-residual-bound}
 \|\xi(t)\|_2^2\ge2\lambda_+L(t).
\end{equation}
By the definition of $L(t)$, we have
\[
\begin{aligned}
2L(t)
&=\frac1m\sum_{i=1}^m
  \bigl(b_i^\top(x^\star-x(t))\bigr)^2\\
&=\bigl(x^\star-x(t)\bigr)^\top
  \left(\frac1m\sum_{i=1}^m b_i b_i^\top\right)
  \bigl(x^\star-x(t)\bigr).
\end{aligned}
\]
Using \eqref{eq:dissipation}, we conclude that
$M(t)\succeq\mu I_{\min\{d_1, d_2\}}$ implies
$\frac d{dt}L(t)\le-2\mu\lambda_+L(t)$.
The desired result follows from integration.
\end{proof}

\begin{remark}
    For comparison, let $X_{\mathrm{dir}}(t)$ denote gradient flow for
$\ell(X)$ with respect to $X$, initialized at $X_{\mathrm{dir}}(0)=0$.
Its velocity is a linear combination of the matrices $A_i$, so
\eqref{eq:aligned-measurements} implies that it remains aligned.
Define the vector $x_{\mathrm{dir}}(t)\in
\mathbb R^{\min\{d_1, d_2\}}$ by
\[
 X_{\mathrm{dir}}(t)
 =\Phi\operatorname{diag}(x_{\mathrm{dir}}(t))\Psi^\top.
\]
The training dynamics of $x_{\mathrm{dir}}(t)$ are
\[
 \dot x_{\mathrm{dir}}(t)
 =\frac1m\sum_{i=1}^m
 b_i\bigl(y_i-b_i^\top x_{\mathrm{dir}}(t)\bigr).
\]
Thus $x_{\mathrm{dir}}(t)\in
\operatorname{span}\{b_i:i\in[m]\}$.

In contrast, the SNN dynamics from
Lemma~\ref{prop:contraction} are
\[
 \begin{aligned}
 \dot x(t)
 &=M(t)\frac1m\sum_{i=1}^m
 b_i\bigl(y_i-b_i^\top x(t)\bigr).
 \end{aligned}
\]
By the expression \eqref{eq:output-flow}, this extra factor magnifies the larger coordinates, and this is the drive for implicit bias towards minimizing nuclear norm.
\end{remark}

By Lemma~\ref{prop:contraction}, to show convergence, we need to prove $M(t)\succeq\mu I_{\min\{d_1,d_2\}}$ for some $\mu>0$. To this end, we derive the following bounds for the parameters during training.

\begin{lemma}
\label{lem:aligned-boundedness}
Under \eqref{eq:aligned-measurements}, \eqref{eq:aligned-init}, and using initialization\begin{equation}\label{eq:active-initialization}
 \alpha_k(0)=1,\qquad u_k(0)=\varepsilon\mathbf1_{\min\{d_1, d_2\}},\qquad v_k(0)=0
 \quad(1\le k\le K)
\end{equation} with $\varepsilon>0$ and $\gamma$ satisfying \eqref{eq:activation-family},
the gradient flow is global and the parameters satisfy the following.

(i) For every $k,j$ and $t\ge0$, we have
\[
 \alpha_k(t)^2\ge1,\qquad \alpha_k(t)\ge1.
\]

(ii) For every $k,j$ and $t\ge0$, we have
\[
 \begin{aligned}
 u_{kj}(t)^2-v_{kj}(t)^2&=\varepsilon^2,\\
 u_{kj}(t)^2+v_{kj}(t)^2
 &=\sqrt{\varepsilon^4+4z_{kj}(t)^2}
 \ge\varepsilon^2.
 \end{aligned}
\]

(iii) There is a finite constant $B\ge0$
such that, for every $k,j$ and $t\ge0$, we have
\[
 |z_{kj}(t)|\le B,\qquad
 \gamma'(z_{kj}(t))^2
 \ge\min_{|s|\le B}\gamma'(s)^2>0.
\]
All parameters are bounded, and their upper bounds can be chosen
uniformly for $0<\varepsilon\le1$.
\end{lemma}

\begin{proof}
Along any solution of \eqref{eq:gradient-flow}, by chain rule and
the definition of the gradient, we have
\[
 \begin{aligned}
 \frac d{dt}L(t)
 &=DL(\alpha(t),U(t),V(t))
   [(\dot\alpha(t),\dot U(t),\dot V(t))]\\
 &=\left\langle
   \operatorname{grad}L(\alpha(t),U(t),V(t)),
   (\dot\alpha(t),\dot U(t),\dot V(t))
   \right\rangle\\
 &=-\|(\dot\alpha(t),\dot U(t),\dot V(t))\|^2.
 \end{aligned}
\]
Therefore, on every finite interval where the solution exists, we have
\begin{equation}\label{eq:finite-time-length}
 \begin{aligned}
 \int_0^T\|(\dot\alpha(t),\dot U(t),\dot V(t))\|\,dt
 &\le\sqrt{T\int_0^T
       \|(\dot\alpha(t),\dot U(t),\dot V(t))\|^2\,dt}\\
 &\le\sqrt{T L(0)}.
 \end{aligned}
\end{equation}
Therefore, by smooth ODE continuation, we can also show global
existence (see, e.g., \cite[Corollary~2.15]{teschl2012}).

Next, we bound $\alpha(t)$ by \eqref{eq:coordinate-balance}. 
By initialization and symmetry, we have
\[
 \alpha_k(t)=\alpha_1(t),\qquad
 u_k(t)=u_1(t),\qquad
 v_k(t)=v_1(t),
 \qquad
 x(t)=K\alpha_1(t)\gamma(z_1(t)).
\]
By \eqref{eq:coordinate-balance}, we have
\[
 u_{1j}(t)^2-v_{1j}(t)^2=\varepsilon^2.
\]
Since $z_{1j}(t)=u_{1j}(t)v_{1j}(t)$, using the identity
$(a^2+b^2)^2=(a^2-b^2)^2+4a^2b^2$, we calculate
\[
 u_{1j}(t)^2+v_{1j}(t)^2
 =\sqrt{\varepsilon^4+4z_{1j}(t)^2}
 \ge\varepsilon^2.
\]
This proves (ii), and \eqref{eq:coordinate-balance} becomes
\[
 \dot z_{1j}(t)
 =
 \alpha_1(t)\gamma'(z_{1j}(t))
 \sqrt{\varepsilon^4+4z_{1j}(t)^2}\,\xi_j(t).
\]

We next express $\alpha_1(t)^2$ in terms of $z_1(t)$.
Since
\[
 \frac d{dt}\alpha_1(t)^2
 =
 2\alpha_1(t)
 \sum_{j=1}^{\min\{d_1, d_2\}}
 \gamma(z_{1j}(t))\xi_j(t),
\]
we choose $F_\varepsilon$ so that its derivatives
will produce the summands on the right. We also define
$Q_\varepsilon$, whose derivative will produce
$\alpha_1(t)\xi_j(t)$ for later use
\begin{equation}\label{eq:activation-primitives}
 \begin{aligned}
 F_\varepsilon(z)
 &=2\int_0^z
   \frac{\gamma(s)}
        {\gamma'(s)\sqrt{\varepsilon^4+4s^2}}\,ds.
 \end{aligned}
\end{equation}
By substitution of $\dot z_{1j}(t)$ into the chain rule, we have
\[
 \begin{aligned}
 \frac d{dt}F_\varepsilon(z_{1j}(t))
 &=F_\varepsilon'(z_{1j}(t))\dot z_{1j}(t)
 =2\alpha_1(t)\gamma(z_{1j}(t))\xi_j(t).
 \end{aligned}
\]
Summing this identity over $j$ and using \eqref{eq:aligned-flow},
we obtain
\[
 \begin{aligned}
 \frac d{dt}\sum_{j=1}^{\min\{d_1, d_2\}}
 F_\varepsilon(z_{1j}(t))
 &=2\alpha_1(t)\sum_{j=1}^{\min\{d_1, d_2\}}
   \gamma(z_{1j}(t))\xi_j(t)\\
 &=2\alpha_1(t)\dot\alpha_1(t)
 =\frac d{dt}\alpha_1(t)^2.
 \end{aligned}
\]
By integrating from $0$ to $t$, we have
\[
 \sum_{j=1}^{\min\{d_1, d_2\}}
 \bigl(F_\varepsilon(z_{1j}(t))
       -F_\varepsilon(z_{1j}(0))\bigr)
 =\alpha_1(t)^2-\alpha_1(0)^2.
\]
By \eqref{eq:active-initialization} and \eqref{eq:coordinates},
we have $\alpha_1(0)=1$ and $z_{1j}(0)=0$.
Since \eqref{eq:activation-primitives} gives $F_\varepsilon(0)=0$,
we conclude that
\begin{equation}\label{eq:unequal-coefficient-balance}
 \alpha_1(t)^2
 -\sum_{j=1}^{\min\{d_1, d_2\}}
 F_\varepsilon(z_{1j}(t))
 =1.
\end{equation}
The function $F_\varepsilon$ is even and nonnegative.
So we have $\alpha_1(t)^2\ge1$, and continuity together with
$\alpha_1(0)=1$ implies $\alpha_1(t)\ge1$. This proves (i).

Next, we bound $x(t)$ by the labels $y$. We assume $y \ne 0$, because otherwise it is trivial.

Let $\Pi$ be the orthogonal projection onto
$\operatorname{span}\{b_i:i\in[m]\}$, and let $\lambda_+>0$
be the smallest positive eigenvalue of
$m^{-1}\sum_i b_i b_i^\top$, as in
Lemma~\ref{prop:contraction}.
Loss decrease and the triangle inequality give
\[
 \begin{aligned}
 \left(\frac1m\sum_{i=1}^m
       (b_i^\top x(t))^2\right)^{1/2}
 &\le \frac{\|y\|_2}{\sqrt m}+\sqrt{2L(t)}\\
 &\le \frac{2\|y\|_2}{\sqrt m}.
 \end{aligned}
\]
The matrix $m^{-1}\sum_i b_i b_i^\top$ has eigenvalues at least $\lambda_+$
on the range of $\Pi$, so
\[
 \lambda_+\|\Pi x(t)\|_2^2
 \le \frac1m\sum_{i=1}^m(b_i^\top x(t))^2.
\]
Therefore, we have
\[
 \|\Pi x(t)\|_2
 \le\frac{2\|y\|_2}{\sqrt{m\lambda_+}}.
\]

To turn this into a bound on the whole vector $x(t)$, we define
\begin{equation}\label{eq: defQ} Q_\varepsilon(z)
 :=\int_0^z
   \frac{ds}{\gamma'(s)\sqrt{\varepsilon^4+4s^2}}\end{equation} and 
\[q(t):=Q_\varepsilon(z_1(t)).\]
By chain rule, we have
\[\frac d{dt}Q_\varepsilon(z_{1j}(t))
 =Q_\varepsilon'(z_{1j}(t))\dot z_{1j}(t)
 =\alpha_1(t)\xi_j(t)\]
Therefore, we have
\[
 \begin{aligned}
 q(t)
 &=\int_0^t\alpha_1(s)\xi(s)\,ds\\
 &=\frac1m\sum_{i=1}^m b_i
   \int_0^t\alpha_1(s)
   \bigl(y_i-b_i^\top x(s)\bigr)\,ds.
 \end{aligned}
\]
In particular, we have $\Pi q(t)=q(t)$.

The functions $Q_\varepsilon$ and $\gamma$ are odd and strictly
increasing, so $Q_\varepsilon(0)=\gamma(0)=0$.
Recall by the definition of $q(t)$ and \eqref{eq:coordinates}, we have
\[
 q_j(t)=Q_\varepsilon(z_{1j}(t)),
 \qquad
 x_j(t)=K\alpha_1(t)\gamma(z_{1j}(t)).
\]
Since $K\alpha_1(t)>0$, strict monotonicity gives
\[
 \operatorname{sgn}(q_j(t))
 =\operatorname{sgn}(z_{1j}(t))
 =\operatorname{sgn}(x_j(t)),
\]
i.e., $q_j(t)$ and $x_j(t)$ have the same sign.
By symmetry of initialization, i.e., $u_k(0)=\varepsilon\mathbf1_{\min\{d_1, d_2\}}$, and the training dynamics, an index $j$ maximizing $|z_{1j}(t)|$
also maximizes $|z_{kj}(t)|$ for any $k$, so both $|q_j(t)|$ and $|x_j(t)|$ are also maximized. Therefore, we have
\begin{equation}\label{eq:aligned-pairing-bound}
 \langle q(t),x(t)\rangle
 =
 \sum_{j=1}^{\min\{d_1, d_2\}}
 |q_j(t)|\,|x_j(t)|
 \ge
 \|q(t)\|_\infty\|x(t)\|_\infty.
\end{equation}
By Cauchy--Schwarz, we also have
\[
 \begin{aligned}
 \langle q(t),x(t)\rangle
 &=\langle q(t),\Pi x(t)\rangle\\
 &\le\|q(t)\|_2\|\Pi x(t)\|_2\\
 &\le\sqrt{\min\{d_1, d_2\}}
       \|q(t)\|_\infty\|\Pi x(t)\|_2.
 \end{aligned}
\]
If $q(t)\ne0$, dividing by $\|q(t)\|_\infty$, we obtain
\begin{equation}\label{eq:aligned-output-bound}
 \|x(t)\|_\infty
 \le\sqrt{\min\{d_1, d_2\}}\,\|\Pi x(t)\|_2
 \le
 2\sqrt{\frac{\min\{d_1, d_2\}}{m\lambda_+}}\,\|y\|_2.
\end{equation}
If $q(t)=0$, by strict monotonicity of $Q_\varepsilon$, we have
$z_1(t)=0$, and therefore $x(t)=0$, so the same bound holds.

Finally, we bound other parameters by $x(t)$.
For $s=|z_{1j}(t)|$, by
\eqref{eq:unequal-coefficient-balance}, we have
 \begin{align}\label{eq:xlower}
 |x_j(t)|
 &=K\alpha_1(t)\gamma(s)\\
 &\ge K\gamma(s)\sqrt{1+F_\varepsilon(s)}.\nonumber
 \end{align}
The left side is bounded by \eqref{eq:aligned-output-bound}.
Therefore, $|z_{1j}(t)|$ will be bounded once we show
\[
 \gamma(s)\sqrt{1+F_\varepsilon(s)}
 \rightarrow\infty
\]
as $s\to\infty$.

For each fixed $\varepsilon>0$, by Cauchy--Schwarz, we have
\[
 \begin{aligned}
 \gamma(R)F_\varepsilon(R)
 &=
 \left(\int_0^R\gamma'(s)\,ds\right)
 \left(\int_0^R F_\varepsilon'(s)\,ds\right)\\
 &\ge
 \left(\int_0^R
       \sqrt{\gamma'(s)F_\varepsilon'(s)}\,ds\right)^2\\
 &=
 \left(\int_0^R
       \sqrt{\frac{2\gamma(s)}
                    {\sqrt{\varepsilon^4+4s^2}}}\,ds\right)^2.
 \end{aligned}
\]
For $s\ge\max\{1,\varepsilon^2\}$, we have
\[
 \sqrt{\frac{2\gamma(s)}
              {\sqrt{\varepsilon^4+4s^2}}}
 \ge
 \sqrt{\frac{2\gamma(1)}{\sqrt5\,s}}.
\]
The integral of $s^{-1/2}$ diverges, so
$\gamma(R)F_\varepsilon(R)\to\infty$.
Since $\gamma(R)\ge\gamma(1)>0$ for $R\ge1$, this also gives
$\gamma(R)\sqrt{1+F_\varepsilon(R)}\to\infty$.
Therefore all $|z_{1j}(t)|$ are bounded for each fixed
$\varepsilon>0$.

For $0<\varepsilon\le1$, we have
$F_\varepsilon(s)\ge F_1(s)$ for $s\ge0$.
Choose $B>0$ such that
\[
 K\gamma(B)\sqrt{1+F_1(B)}
 >
 2\sqrt{\frac{\min\{d_1, d_2\}}{m\lambda_+}}\,\|y\|_2.
\]
The bound \eqref{eq:xlower}  then implies
$|z_{1j}(t)|\le B$ for every $t$ and every
$0<\varepsilon\le1$.

Finally, for any $\varepsilon>0$ and any corresponding bound $B$
on $|z_{1j}(t)|$, we have
\[
 F_\varepsilon(s)
 \le
 \int_0^s\frac{\gamma(v)}{v\gamma'(v)}\,dv
 \qquad(s\ge0).
\]
The integrand extends continuously to $v=0$ with value $1$.
We also have
\[
 \begin{aligned}
 \alpha_1(t)^2
 &\le
 1+\min\{d_1, d_2\}
 \int_0^B\frac{\gamma(v)}{v\gamma'(v)}\,dv,\\
 u_{1j}(t)^2+v_{1j}(t)^2
 &\le\sqrt{\varepsilon^4+4B^2}.
 \end{aligned}
\]
These bounds are finite for every fixed $\varepsilon>0$.
For $0<\varepsilon\le1$, the same $B$ works and
$\varepsilon^4\le1$, so both bounds are uniform.

Parameters for all $k$ are bounded by the same argument.
Finally, continuity and strict positivity of $\gamma'$ imply
\[
 \gamma'(z_{kj}(t))^2
 \ge\min_{|s|\le B}\gamma'(s)^2>0,
\]
which proves (iii).
\end{proof}

These bounds on the parameters in Lemma~\ref{lem:aligned-boundedness} imply
$M(t)\succeq\mu I_{\min\{d_1, d_2\}}$ for
$\mu=K\varepsilon^2\min_{|s|\le B}\gamma'(s)^2>0$. We also need the next lemma to show that the gradient flow will become nearly stationary when $L(t)$ approaches $0$.

\begin{lemma}\label{lem:gradient-length} Let $b>0$ and $\nu<1$. 
If the flow \eqref{eq:gradient-flow} satisfies
\[
 \|\operatorname{grad}L(\alpha(t),U(t),V(t))\|
 \ge bL(t)^\nu \quad \text{and} \quad L(t)>0
\]
for any $t_0\le t\le T$,
then we have
\begin{equation}\label{eq:gradient-length}
 \int_{t_0}^T\|(\dot\alpha(t),\dot U(t),\dot V(t))\|\,dt
 \le\frac{L(t_0)^{1-\nu}-L(T)^{1-\nu}}{b(1-\nu)}.
\end{equation}
\end{lemma}
\begin{proof}
Since $L(t)>0$, by a direct calculation, we have
\begin{align*}
 -\frac d{dt}L(t)^{1-\nu}
 &=(1-\nu)\frac{\|\operatorname{grad}L(\alpha(t),U(t),V(t))\|^2}{L(t)^\nu}
 \\
 &\ge b(1-\nu)\|\operatorname{grad}L(\alpha(t),U(t),V(t))\|.
\end{align*}
Using \eqref{eq:gradient-flow}, dividing by $b(1-\nu)>0$,
and integrating over $[t_0,T]$, we obtain
\[
\begin{aligned}
 \int_{t_0}^T\|(\dot\alpha(t),\dot U(t),\dot V(t))\|\,dt
 &=
 \int_{t_0}^T
 \|\operatorname{grad}L(\alpha(t),U(t),V(t))\|\,dt\\
 &\le
 -\frac{1}{b(1-\nu)}
 \int_{t_0}^T\frac d{dt}L(t)^{1-\nu}\,dt\\
 &=
 -\frac{1}{b(1-\nu)}
 \bigl[L(t)^{1-\nu}\bigr]_{t=t_0}^{t=T}\\
 &=
 \frac{L(t_0)^{1-\nu}-L(T)^{1-\nu}}{b(1-\nu)}.
\end{aligned}
\]
This proves \eqref{eq:gradient-length}.
\end{proof}

With the above preparations, we can show the convergence of the loss and the parameters.

\begin{theorem}\label{lem:aligned-parameter-convergence}
Under \eqref{eq:aligned-measurements}, \eqref{eq:aligned-init},
and \eqref{eq:active-initialization}, with $\gamma$ satisfying
\eqref{eq:activation-family}, the gradient flow is defined for every $t\ge0$,
and all parameters converge to finite limits. In particular,
\[
 X_\infty^\varepsilon
 :=\lim_{t\to\infty}
 X_{\rm train}(\alpha(t),U(t),V(t))
\]
exists and satisfies $\ell(X_\infty^\varepsilon)=0$.
\end{theorem}
\begin{proof}
If $y_i=0$ for every $i\in[m]$, the initialization has zero loss,
so the gradient flow is stationary and the result follows.
Otherwise, $m^{-1}\sum_i b_i b_i^\top\ne0$.
Lemma~\ref{lem:aligned-boundedness} implies
$M(t)\succeq\mu I_{\min\{d_1, d_2\}}$, where
$\mu=K\varepsilon^2\min_{|s|\le B}\gamma'(s)^2>0$.
By Lemma~\ref{prop:contraction}, we have
\[
 L(t)\le L(0)e^{-2\mu\lambda_+t},
 \qquad
 -\frac d{dt}L(t)
 =\bigl\|(\dot\alpha(t),\dot U(t),\dot V(t))\bigr\|^2
 \ge2\mu\lambda_+L(t).
\]
Applying Lemma~\ref{lem:gradient-length} with $\nu=1/2$ and
$b=\sqrt{2\mu\lambda_+}$ on any interval $[0,T]$ where $L(t)>0$,
we obtain
\begin{equation}\label{eq:parameter-length-finite}
	\int_0^T
	\bigl\|(\dot\alpha(t),\dot U(t),\dot V(t))\bigr\|\,dt
	\le \sqrt{\frac{2}{\mu\lambda_+}}
	\bigl(\sqrt{L(0)}-\sqrt{L(T)}\bigr)
	\le \sqrt{\frac{2L(0)}{\mu\lambda_+}}.
\end{equation}
If $L(t)>0$ for all $t\ge0$, since the integrand is nonnegative, the left-hand side increases
to the integral over $[0,\infty)$ as $T\to\infty$. Therefore, we have
\[
 \int_0^\infty
 \bigl\|(\dot\alpha(t),\dot U(t),\dot V(t))\bigr\|\,dt
 \le\sqrt{\frac{2L(0)}{\mu\lambda_+}}.
\]
Otherwise, let $\tau>0$ be the first time at which $L(\tau)=0$.
By \eqref{eq:loss} and \eqref{eq:dissipation}, we have
\[
0\le L(t)\le L(\tau)=0
\qquad(t\ge\tau).
\]
In conclusion, in either case,
\[
\int_0^\infty
\bigl\|(\dot\alpha(t),\dot U(t),\dot V(t))\bigr\|\,dt
\le \sqrt{\frac{2L(0)}{\mu\lambda_+}}.
\]
So all
parameters converge to finite limits.
Continuity and $L(t)\to0$ imply
\[\ell(X_\infty^\varepsilon)=0.\]
\end{proof}

By Theorem~\ref{lem:aligned-parameter-convergence}, write
$x_\infty^\varepsilon=\lim_{t\to\infty}x(t)$. Then
\begin{equation}\label{eq:aligned-endpoint-bound}
 X_\infty^\varepsilon
 =\Phi\operatorname{diag}(x_\infty^\varepsilon)\Psi^\top,
 \qquad
 b_i^\top x_\infty^\varepsilon=y_i
 \quad(i\in[m]).
\end{equation}

It remains to determine which zero loss solution the dynamics select.
To prove implicit bias, we will show that, for every
$\widetilde x$ satisfying $b_i^\top\widetilde x=y_i$ for all $i\in[m]$,
\[
 \|x_\infty^\varepsilon\|_1
 \le \|\widetilde x\|_1+o(1)
 \qquad \text{ when }\varepsilon\to0.
\]
Lemma~\ref{lem:aligned-dual-certificate} provides the quantitative
estimate used to prove this comparison in
Theorem~\ref{thm:implicit-regularization}.

\begin{lemma}\label{lem:aligned-dual-certificate}
Under the hypotheses of Theorem~\ref{lem:aligned-parameter-convergence},
choose $B>0$ bounding all $|z_{1j}(t)|$ uniformly for $t\ge0$ and
$0<\varepsilon\le1$. For these $t,\varepsilon$, we 
have\[
 \widehat q_\varepsilon(t)
 :=\frac{Q_\varepsilon(z_1(t))}{Q_\varepsilon(B)}
 \in\operatorname{span}\{b_i:i\in[m]\},
 \qquad\|\widehat q_\varepsilon(t)\|_\infty\le1.
\]
 In addition, we have
\begin{equation}\label{eq:aligned-approx-subgradient}
 \|\widetilde x\|_1
 \ge\|x(t)\|_1+
 \langle\widehat q_\varepsilon(t),\widetilde x-x(t)\rangle
 -\frac{C}{Q_\varepsilon(B)}
\end{equation}
for every $\widetilde x\in\mathbb R^{\min\{d_1, d_2\}}$, where $C$ is
independent of $t,\varepsilon$ and $Q_\varepsilon(B)\to\infty$ as
$\varepsilon\to0$.
\end{lemma}
\begin{proof}
The uniform bounds on $|z_{1j}(t)|$ and that 
$\widehat q_\varepsilon(t)\in\operatorname{span}\{b_i:i\in[m]\}$
follow from Lemma~\ref{lem:aligned-boundedness} and the definition
of $\widehat q_\varepsilon(t)$, i.e., $\widehat q_\varepsilon(t)
 =\frac{Q_\varepsilon(z_1(t))}{Q_\varepsilon(B)}$.
By monotonicity of $Q_\varepsilon$, we have $\|\widehat q_\varepsilon(t)\|_\infty\le1$.
To prove \eqref{eq:aligned-approx-subgradient}, we bound the difference
$\|x(t)\|_1-\langle\widehat q_\varepsilon(t),x(t)\rangle$.
For $0<s\le B$, continuity and positivity of $\gamma'$ on $[0,B]$ give
\[
 \begin{aligned}
 0\le\gamma(s)\bigl[Q_\varepsilon(B)-Q_\varepsilon(s)\bigr]
 &\le\gamma(s)\int_s^B\frac{du}{2u\gamma'(u)}\\
 &\le C s\log(B/s)\le C.
 \end{aligned}
\]
At $s=0$ the original product is zero.
Using $x(t)=K\alpha_1(t)\gamma(z_1(t))$ and the uniform bound on
$\alpha_1(t)$, we obtain
\begin{equation}\label{eq:aligned-norm-gap}
\begin{aligned}
 0&\le\|x(t)\|_1-\langle\widehat q_\varepsilon(t),x(t)\rangle\\
 &=\frac{K\alpha_1(t)}{Q_\varepsilon(B)}
 \sum_{j=1}^{\min\{d_1, d_2\}}\gamma(|z_{1j}(t)|)
 \bigl[Q_\varepsilon(B)-Q_\varepsilon(|z_{1j}(t)|)\bigr]
 \le\frac{C}{Q_\varepsilon(B)}.
\end{aligned}
\end{equation}
Now $\|\widetilde x\|_1\ge
\langle\widehat q_\varepsilon(t),\widetilde x\rangle$
proves \eqref{eq:aligned-approx-subgradient}.
By monotone convergence and $\gamma'(0)=1$, we have
\begin{equation}\label{eq:aligned-normalized-limit}
 Q_\varepsilon(B)\rightarrow
 \int_0^B\frac{ds}{2s\gamma'(s)}=\infty
 \qquad(\varepsilon\to0).
\end{equation}
\end{proof}

\begin{theorem}\label{thm:implicit-regularization}
Under \eqref{eq:aligned-measurements}, \eqref{eq:aligned-init},
and \eqref{eq:active-initialization}, with $\gamma$ satisfying \eqref{eq:activation-family} we have that every accumulation point $X_\infty$ of $\{X_\infty^\varepsilon\}_{0<\varepsilon<1}$ as
$\varepsilon \to 0$ solves
\begin{equation}\label{eq:nuclear-norm-opt}
 \min_{X\in\mathbb R^{d_1\times d_2}}\|X\|_*
 \quad\textrm{subject to}\quad
 y_i=\langle A_i,X\rangle
 \quad\forall i\in[m].
\end{equation}

\end{theorem}

\begin{proof}
By Lemma~\ref{lem:aligned-boundedness}, we know that the family $\{X_\infty^\varepsilon:0<\varepsilon<1\}$
is bounded.

Fix $0<\varepsilon\le1$, and let $B>0$ be as in
Lemma~\ref{lem:aligned-dual-certificate}.
By Theorem~\ref{lem:aligned-parameter-convergence},
$u_1(t)$ and $v_1(t)$ converge to finite limits.
Therefore $z_1(t)=u_1(t)\odot v_1(t)$ also converges.
The definition \eqref{eq: defQ} shows that
$Q_\varepsilon$ is continuous and $Q_\varepsilon(B)>0$.
Therefore, we conclude that 
\[
 \widehat q_\varepsilon(t)
 =\frac{Q_\varepsilon(z_1(t))}{Q_\varepsilon(B)}
\]
converges to a finite vector.

For every $t\ge0$, Lemma~\ref{lem:aligned-dual-certificate} shows that
$\widehat q_\varepsilon(t)\in
\operatorname{span}\{b_i:i\in[m]\}$.
This span is a finite-dimensional linear subspace and is therefore
closed, so it also contains the limit. We can therefore define
\begin{equation}\label{eq:aligned-dual-coordinate}
 \widehat q_\varepsilon
 :=\lim_{t\to\infty}\widehat q_\varepsilon(t)
 =\frac{Q_\varepsilon\!\left(\lim_{t\to\infty}z_1(t)\right)}
        {Q_\varepsilon(B)}
 \in\operatorname{span}\{b_i:i\in[m]\}.
\end{equation}
In addition, the bound
$\|\widehat q_\varepsilon(t)\|_\infty\le1$
passes to the limit, so we have
$\|\widehat q_\varepsilon\|_\infty\le1$.

We can write $\widehat q_\varepsilon=\sum_{i=1}^m \lambda_i b_i$.
Then
$\Phi\operatorname{diag}(\widehat q_\varepsilon)\Psi^\top
=\sum_{i=1}^m \lambda_i A_i$.
By zero training loss, we have $b_i^\top x_\infty^\varepsilon=y_i$,
and feasibility of $X$ implies $\langle A_i,X\rangle=y_i$.
Therefore, we have
\[
 \begin{aligned}
 \langle\widehat q_\varepsilon,x_\infty^\varepsilon\rangle
 &=\sum_{i=1}^m \lambda_i b_i^\top x_\infty^\varepsilon
 =\sum_{i=1}^m \lambda_i y_i\\
 &=\sum_{i=1}^m \lambda_i\langle A_i,X\rangle
 =\left\langle
   \Phi\operatorname{diag}(\widehat q_\varepsilon)\Psi^\top,X
   \right\rangle.
 \end{aligned}
\]
 Driving $t\to\infty$ in
\eqref{eq:aligned-norm-gap}, we have
\begin{equation}\label{eq:aligned-approx-optimality}
\begin{aligned}
 \|X_\infty^\varepsilon\|_* =\|x_\infty^\varepsilon\|_1
 &\le\langle\widehat q_\varepsilon,x_\infty^\varepsilon\rangle
       +\frac{C}{Q_\varepsilon(B)}\\
 &=\langle\Phi\operatorname{diag}(\widehat q_\varepsilon)\Psi^\top,X\rangle
       +\frac{C}{Q_\varepsilon(B)}\\
 &\le\|X\|_*+\frac{C}{Q_\varepsilon(B)},
\end{aligned}
\end{equation}
where the last inequality follows from the matrix H\"{o}lder inequality $\langle A, B \rangle \le \|A\|_{op} \|B\|_*$ and the fact that $\|\Phi\operatorname{diag}(\widehat q_\varepsilon)\Psi^\top\| \le \|\widehat q_\varepsilon\|_\infty \le 1$.

Every accumulation point $X_\infty$ as $\varepsilon\to0$ is feasible.
By \eqref{eq:aligned-normalized-limit}, the error above tends to zero,
so $\|X_\infty\|_*\le\|X\|_*$ for every feasible $X$.
\end{proof}

\section{The SNN architecture}
\label{sec:snn-architecture}

The reparameterization \eqref{eq:X} leads to a novel neural network architecture, which we call the Spectral Neural Network (SNN); it is particularly suitable for matrix learning problems. At the fundamental level, the SNN architecture entails an application of a non-linear activation function on a matrix-valued input in the spectral domain. This may be followed by a linear combination of several such spectrally transformed matrix-structured data. 

The SNN architecture consists of elemental spectral neurons that manipulate matrix-valued inputs by spectral activation. More precisely, for $W\in\R^{d_1\times d_2}$ with reduced singular value decomposition
\[
 W=\Phi \, \diag(w_1,\ldots,w_{\min\{d_1, d_2\}})\Psi^\top,
\]
where $\Phi\in\R^{d_1\times \min\{d_1, d_2\}}$ and $\Psi\in\R^{d_2\times \min\{d_1, d_2\}}$
have orthonormal columns and $w_j\ge0$.
For an odd real-analytic scalar activation $\gamma$,  the elemental spectral neuron acts on $W$ by
\[
 \Gamma(W)
 =
 \Phi \, \diag\bigl(\gamma(w_1),\ldots,\gamma(w_{\min\{d_1, d_2\}})\bigr)\Psi^\top
 \in\R^{d_1\times d_2}.
\]
as in Section \ref{sec:SNNreparameterization}.

\begin{figure}[!h]
    \centering
    \includegraphics[width=\columnwidth]{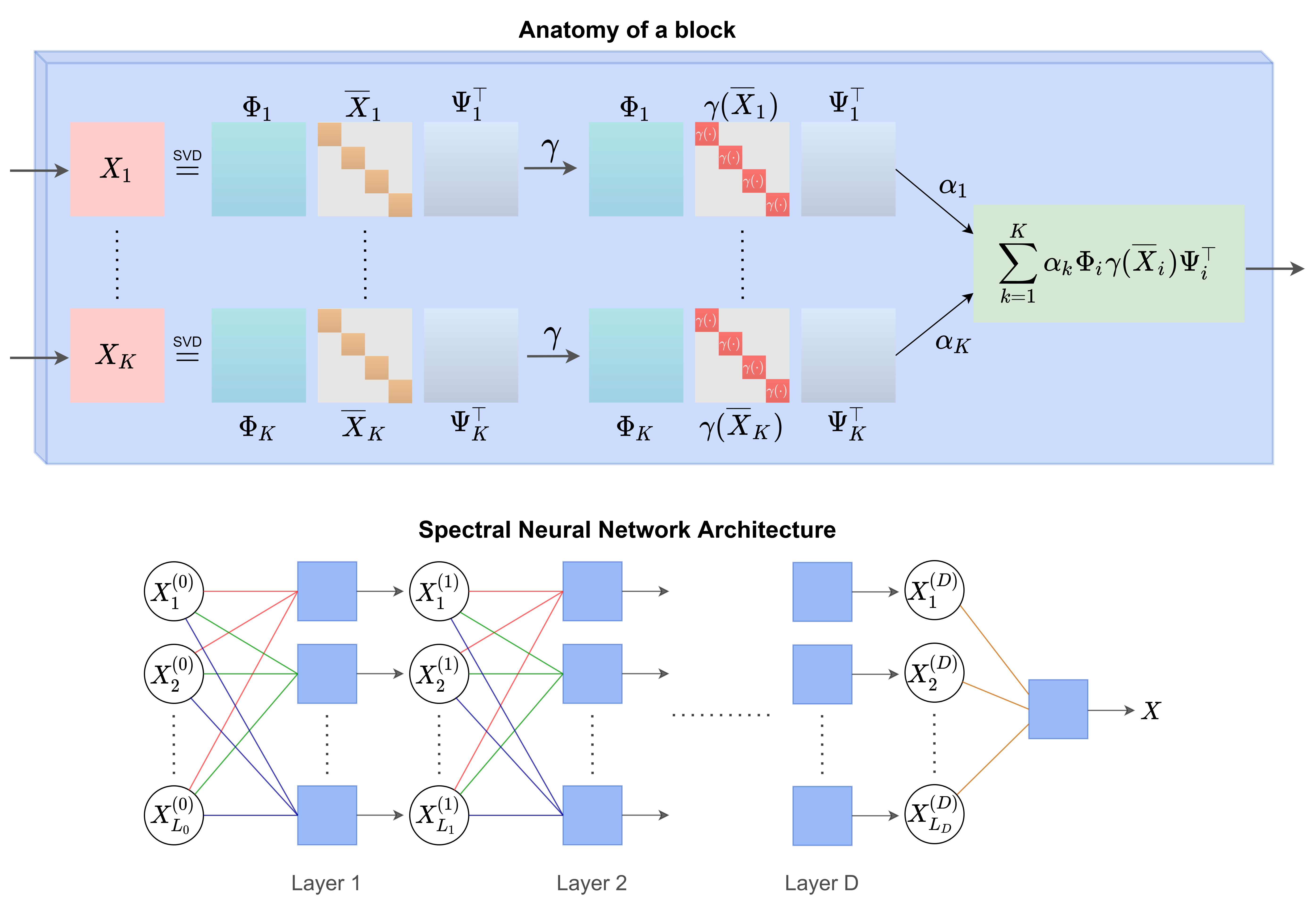}
    \caption{Visualization of the anatomy of an SNN block and a depth-$D$ SNN architecture. Each SNN block takes as input $K$ matrices and outputs one matrix, both the input and output matrices are $d_1 \times d_2$. In layer $i$ of the SNN, there are $L_i$ blocks, which aggregate matrices from the previous layers to produce $L_i$ output matrices as inputs for the next layer. The number of input matrices to a block equals the number of neurons in the previous layer. For example, blocks in layer 1 have $K=L_0$, blocks in layer 2 have $K=L_1$, and blocks in layer $i$ have $K=L_{i-1}$.}
    \label{fig:snn-architecture}
\end{figure}

A \textit{block} in the SNN architecture comprises $K \ge 1$  elemental neurons as above, taking in $K$ matrix-valued inputs $W_1,\ldots,W_K$. Each input matrix $W_i$ is then individually operated upon by an elemental neuron, and finally, the resulting matrices are aggregated linearly to produce a matrix-valued output for the block. The coefficients of this linear combination are also parameters in the SNN architecture, and are to be learned during the process of training the network. 

The comprehensive SNN architecture is finally obtained by combining such blocks into multiple layers of a deep network, as illustrated in Figure~\ref{fig:snn-architecture}. 

\section{Extended Neural Networks}

In this section, we place our SNN in the broader landscape of
extended neural network architectures that operate naturally
in the spectral domain.

\subsection{Neural Operators}

Neural Operators (NO) are designed to learn mappings between
infinite-dimensional function spaces $\mathcal{A}$ and $\mathcal{U}$
defined on a bounded domain $\Omega\subset\mathbb{R}^d$
\cite{kovachki2021neural}.
Unlike a standard multilayer perceptron, which maps
finite-dimensional vectors to vectors, an operator network
approximates a map
\[
    \mathcal{N}:\mathcal{A}\rightarrow\mathcal{U},
    \qquad
    u=\mathcal{N}(a),
\]
whose input $a$ and output $u$ are functions. The aim is therefore to approximate a family of solutions
associated with different input functions, rather than a single
solution for one fixed choice of PDE data
\cite{ray2024deep}.

More precisely, given an input function $a\in\mathcal{A}$, where
$a:\Omega\rightarrow\mathbb{R}$, a pointwise lifting operator
$P:\mathbb{R}\rightarrow\mathbb{R}^{H}$ produces
$u^{(1)}(x)=P(a(x))$.
For $n=1,\ldots,L-1$, the hidden representation
$u^{(n)}:\Omega\rightarrow\mathbb{R}^{H}$ is updated by
\begin{equation}
    u^{(n+1)}(x)
    =
    \sigma\left(
        Wu^{(n)}(x)
        +
        \int_\Omega \kappa(x,y;\phi)u^{(n)}(y)\,dy
    \right).
    \label{eq:no_update}
\end{equation}
Here,
\begin{itemize}
    \item $\sigma:\mathbb{R}^{H}\rightarrow\mathbb{R}^{H}$
    is a fixed non-linear activation applied componentwise;
    \item $W\in\mathbb{R}^{H\times H}$
    is a learnable pointwise linear map;
    \item
    $\kappa(\cdot,\cdot;\phi):
    \Omega\times\Omega\rightarrow\mathbb{R}^{H\times H}$
    is a matrix-valued kernel with learnable parameters $\phi$.
\end{itemize}
The update parameters may differ between layers and their layer
indices are suppressed. Bias terms are also omitted for simplicity.
The final output is $u(x)=Q(u^{(L)}(x))$, where
$Q:\mathbb{R}^{H}\rightarrow\mathbb{R}$
is a pointwise projection.

\subsection{Fourier Neural Operators (FNO)}

For periodic functions on $\Omega=[0,1]^d$, an FNO
parameterizes its linear integral operators by
Fourier multipliers \cite{li2020fourier}.
The input and output functions $a,u:\Omega\rightarrow\mathbb{R}$
are scalar-valued, while the hidden representations
$u^{(n)}:\Omega\rightarrow\mathbb{R}^{H}$ have $H$ components.
Each layer applies the transformation
\begin{equation}
    u^{(n+1)}(x)
    =
    \sigma\left(
        Wu^{(n)}(x)
        +
        \left[\mathcal{F}^{-1}
        \left(R\mathcal{F}u^{(n)}\right)\right](x)
    \right),
\end{equation}
where $W\in\mathbb{R}^{H\times H}$,
$\sigma:\mathbb{R}^{H}\rightarrow\mathbb{R}^{H}$
is a fixed non-linear map, and
\[
    R(k)\in\mathbb{C}^{H\times H},
    \qquad
    (R\mathcal{F}u^{(n)})(k)
    =
    R(k)(\mathcal{F}u^{(n)})(k),
    \qquad k\in\mathbb{Z}^d.
\]
Using the convolution convention $\kappa(y-x)$, we have
$R(k)=\widehat{\kappa}(-k)$, where
$\widehat{\kappa}=\mathcal{F}\kappa$.
The matrices $W$ and $R(k)$ are learned parameters.

\subsection{Connections to SNN}

For this comparison, take $d=1$ and $H=1$, so that
$u^{(n)}$ is scalar-valued and also the Fourier transform is one dimensional. Discretize the periodic domain into $d_1$
equally spaced points, and let
$F\in\mathbb{C}^{d_1\times d_1}$ be the unitary
discrete Fourier transform matrix. Writing $u^{(n)}$
for its vector of sampled values and $R$ for the
vector of discrete Fourier coefficients to be trained, the
convolution term
$\mathcal{F}^{-1}\left(R\mathcal{F}u^{(n)}\right)$
becomes
\[
    F^*\operatorname{diag}(R)Fu^{(n)},
\]
which is a matrix $F^*\operatorname{diag}(R)F$ acting
on a vector $u^{(n)}$, where $F^*$ denotes conjugate
transpose.

Recall that, in SNN, $\Gamma$ transforms a matrix by applying
$\gamma$ to its singular values. A concrete connection
to FNO arises when the matrix represents a periodic
convolution operator. Let $d_1=d_2$ and let $Z$ be a real
circulant matrix. Then it has the representation
\[
    Z
    =
    F^*\operatorname{diag}(r_1,\ldots,r_{d_1})F.
\]
In this case, we have
\[
    \Gamma(Z)
    =
    F^*\operatorname{diag}\left(
        \left(
            \frac{r_j}{|r_j|}\gamma(|r_j|)
        \right)_{j=1}^{d_1}
    \right)F,
\]
where $r_j$ are the trained Fourier coefficients, just like in FNO.

In this setting, SNN applies the nonlinear activation $\gamma$
to the trained parameters, i.e., Fourier coefficients, of the convolution matrix underlying
the FNO term.

\section{Numerical studies}
\label{sec:numerical-studies}

In this section, we present numerical studies to complement our theoretical analysis. 

We highlight that gradient flow can be viewed as gradient descent with an infinitesimal learning rate. Therefore, the gradient flow model serves as a reliable proxy for studying gradient descent when the learning rate is sufficiently small. Throughout our experiments, we shall consider gradient descent with
small learning rates.

The synthetic data used in our experiment is generated by \(X^{\star} = \tilde{X}\tilde{X}^{T} \in\mathbb{R}^{10\times10} \) where each entry of \(\tilde{X}\in\mathbb{R}^{10\times 6} \) is randomly sampled uniformly from \([0,1]\) and normalized so that \( \| X^{\star} \|_{*} =1 \).

\subsection{Experiments violating the aligned measurement assumption} \label{sec:gen_setting} In this experiment, we examine the performance of our model beyond the aligned setting, i.e., when the aligned measurement assumption \eqref{eq:aligned-measurements} does not hold. We show that our model still recovers $X^{\star}$ with reasonable quality, which indicates its effectiveness in solving matrix sensing problems.
Each entry in the measurement matrix $A_i\in\mathbb{R}^{d_1 \times d_2}$ ($i = 1, \dots, m = 50)$ 
is generated by sampling independently from the standard Gaussian distribution, and \(U(0),V(0)\) are initialized as diagonal matrices, whose diagonal entries are \(10^{-4}+10^{-6}\epsilon\) where \(\epsilon\) is sampled from the standard Gaussian distribution. We report and compare the numerical results with the naive linear regression model \(X = \arg\min_{X} \frac{1}{2} \sum_{i=1}^{m} \left(y_i - \langle A_i,X\rangle  \right)^{2} \) and depth-3 factorization model \(X = \arg\min_{W_1,W_2,W_3} \frac{1}{2} \sum_{i=1}^{m} \left(y_i - \langle A_i,W_1W_2W_3\rangle  \right)^{2} \) in Figure~\ref{fig:exp-2} and Figure~\ref{fig:exp-2-X}. In particular, the linear regression and depth-3 factorization models are initialized at \(X(0)=0,W_1(0)=W_2(0)=W_3(0)=10^{-4}I \).

We observe that the linear regression model yields solution matrices with low training errors but their nuclear norm values (\(\approx 1.5\)) and leading singular values are notably different from the ground-truth. This overfitting phenomenon could be partially explained by the setting that the number of observations \(m=50 \) is smaller than the number of variables \(d_1d_2=100\). In contrast, our network \texttt{snn} produces lower nuclear norm values (\(\approx 0.95\)) and leading singular values are closer to the ground-truth. This is consistent with the bias predicted by Theorem~\ref{thm:implicit-regularization}, affirming that our network has a preference for solutions with a lower nuclear norm among all feasible solutions satisfying \(m\) observations constraints. Corresponding to lower nuclear norm values, one also notes that our constructed matrices have better reconstruction quality (\(\texttt{psnr}\approx 20 \)) than the linear regression model (\(\texttt{psnr}\approx 7 \)). This highlights that the regularization phenomenon potentially yields favorable recovery solutions in practice. In addition, we find that the depth-3 factorization model returns solutions with lower effective ranks. However, these outcomes might not be favorable in solving matrix sensing problems as one can see the reconstruction errors are more significant than our model.

\begin{figure}[ht]
\centering
\includegraphics[width=0.8\textwidth]{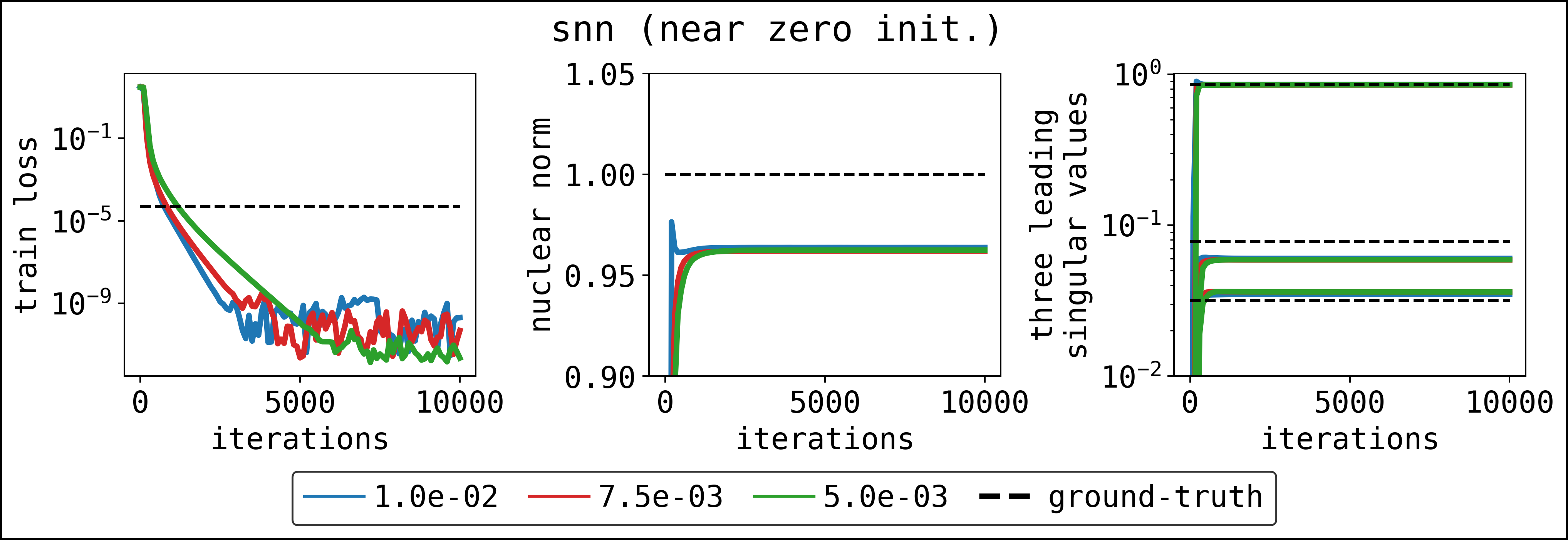}
\includegraphics[width=0.8\textwidth]{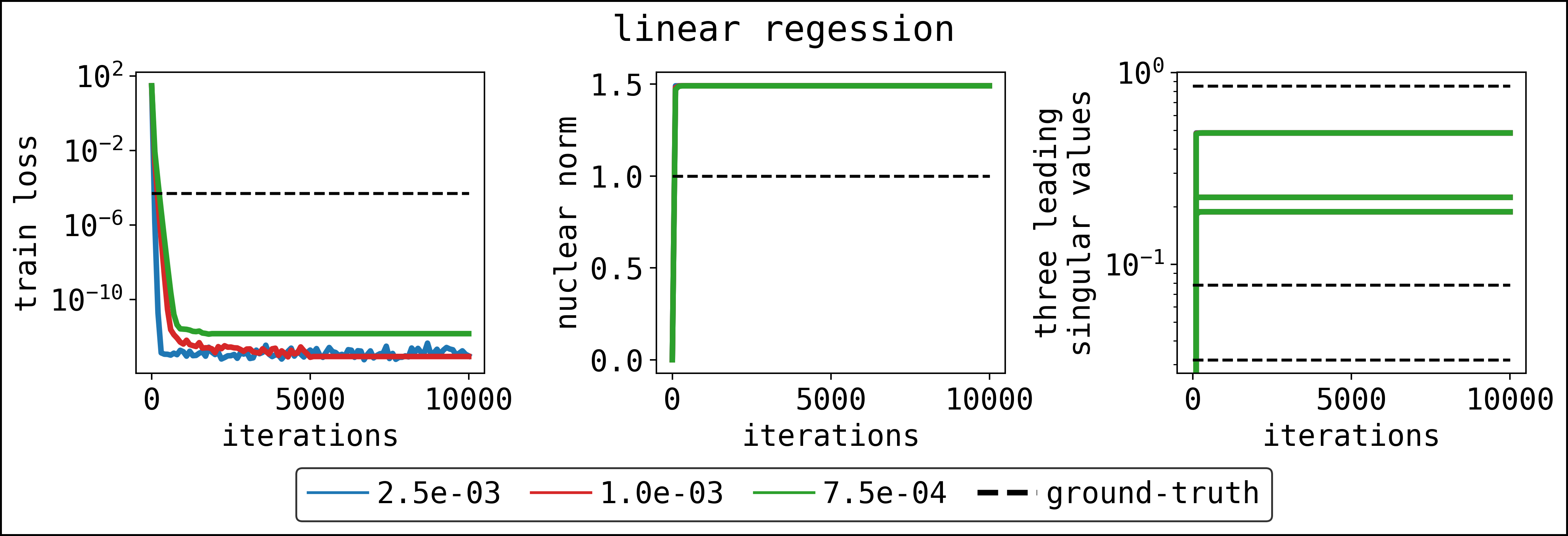}
\includegraphics[width=0.8\textwidth]{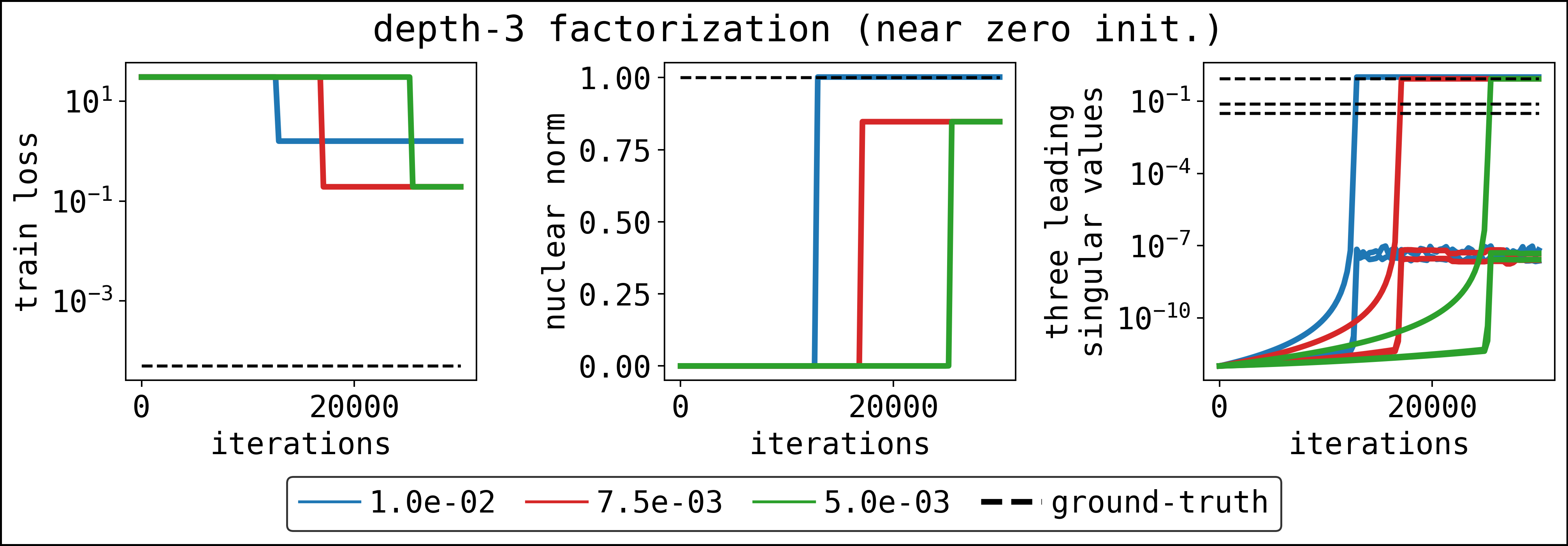}
\caption{The evolution of training error (left column), nuclear norm (middle column), three leading singular values (right column) over time with different learning rates (shown in different colors)
for different models.}
\label{fig:exp-2}
\end{figure}

\begin{figure}[ht]
\centering
\includegraphics[width=0.8\textwidth]{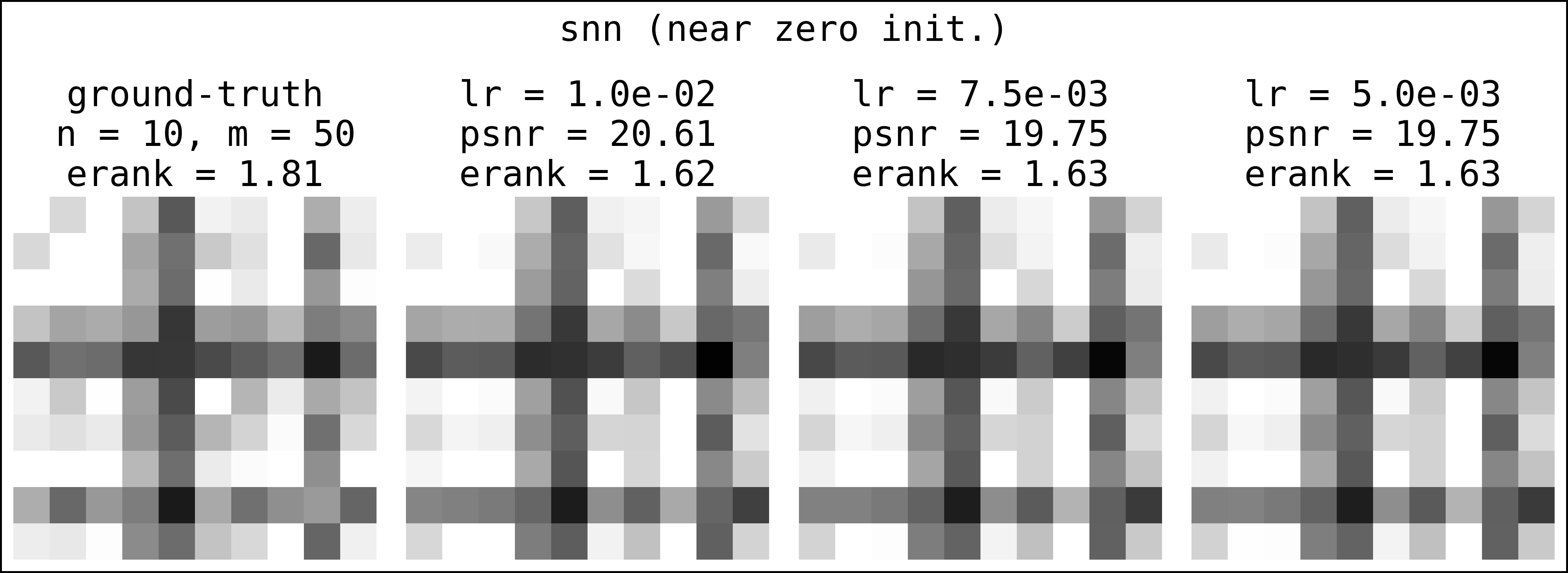}
\includegraphics[width=0.8\textwidth]{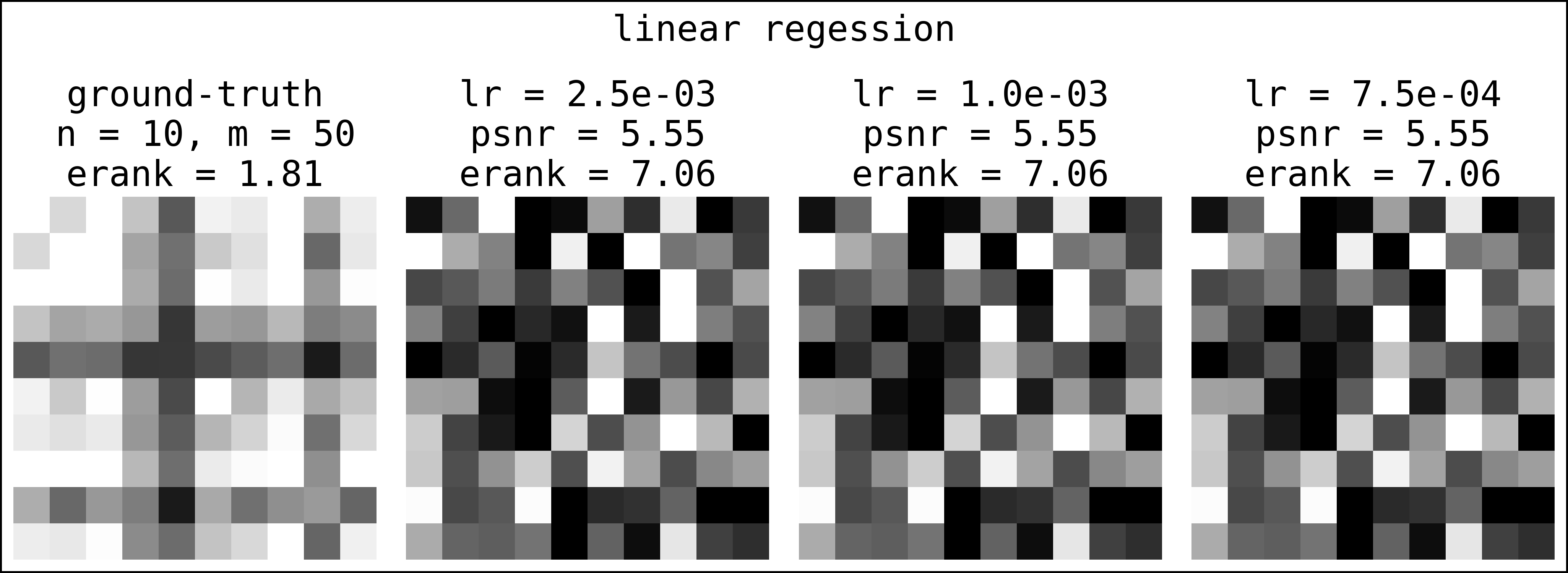}
\includegraphics[width=0.8\textwidth]{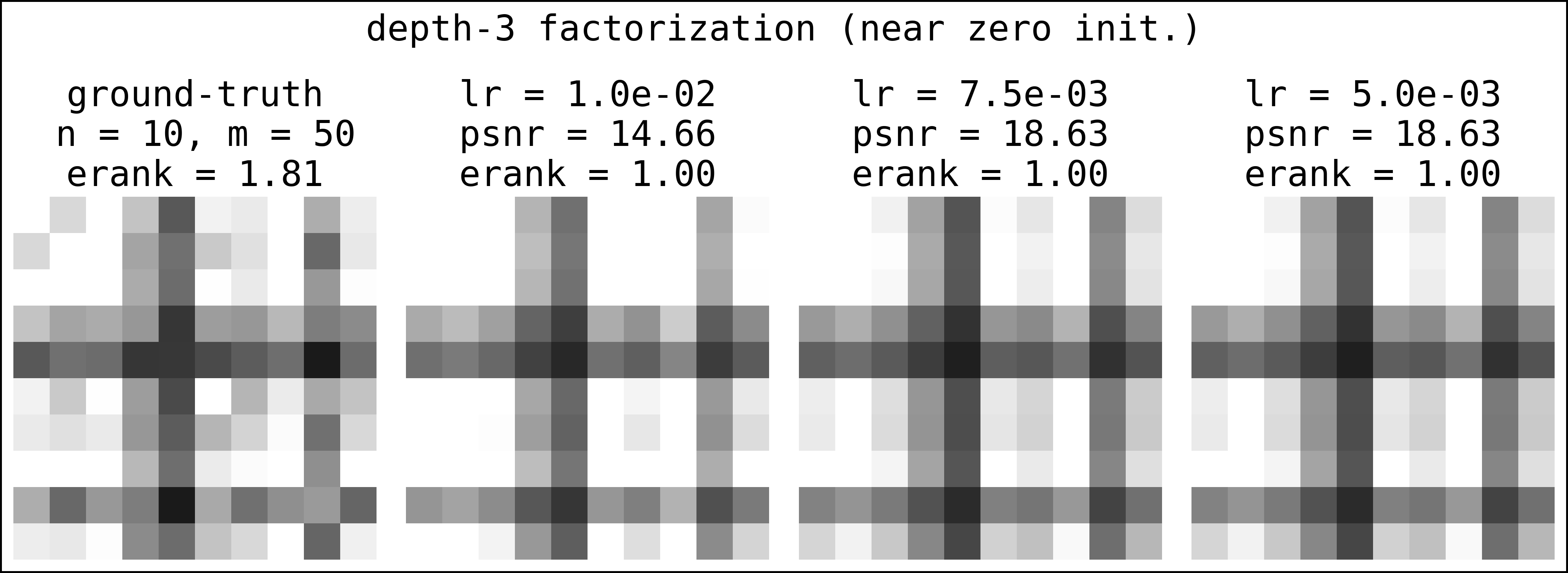}
\caption{The reconstruction images over time with different learning rates (\texttt{lr}) for different models.}
\label{fig:exp-2-X}
\end{figure}

\subsection{Real image reconstruction task} In the next experiment, we illustrate the numerical results where the ground-truth \(X^{\star}\) are real images retrieved from the \texttt{MNIST} dataset. The measurement matrices and initializations are generated as in Section~\ref{sec:gen_setting}, and \(K=4\). The numerical performances reported in Figure~\ref{fig:exp-3} suggest that the spectral regularization phenomena could be regarded as a denoising methodology in certain circumstances where the number of observations is low.

\begin{figure}[ht]
\centering
\includegraphics[width=0.8\textwidth]{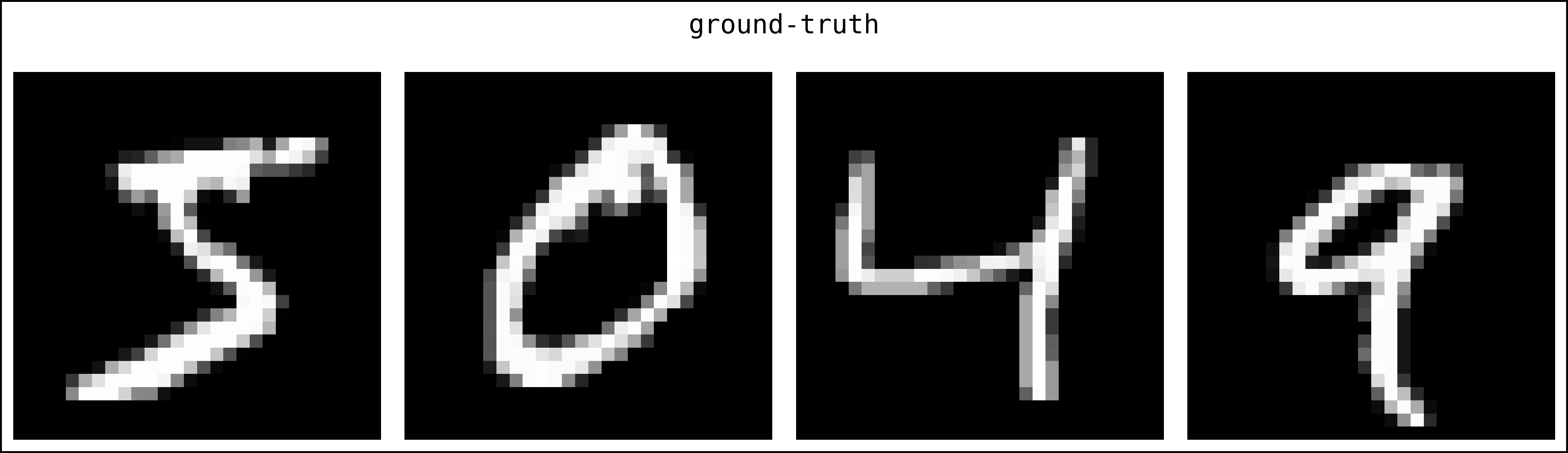}
\includegraphics[width=0.8\textwidth]{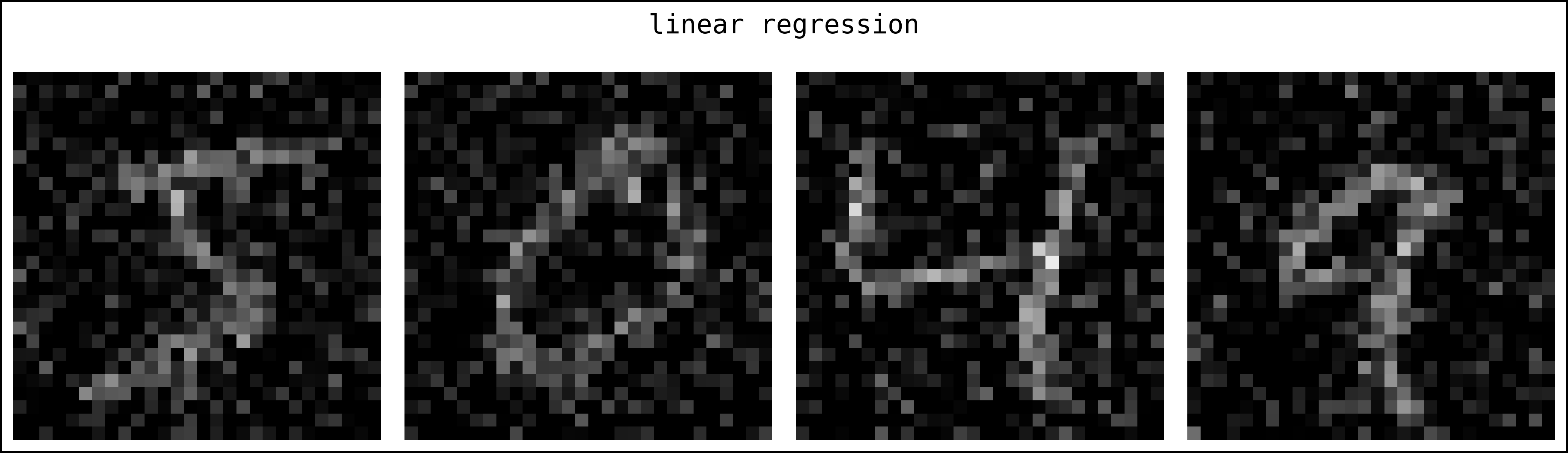}
\includegraphics[width=0.8\textwidth]{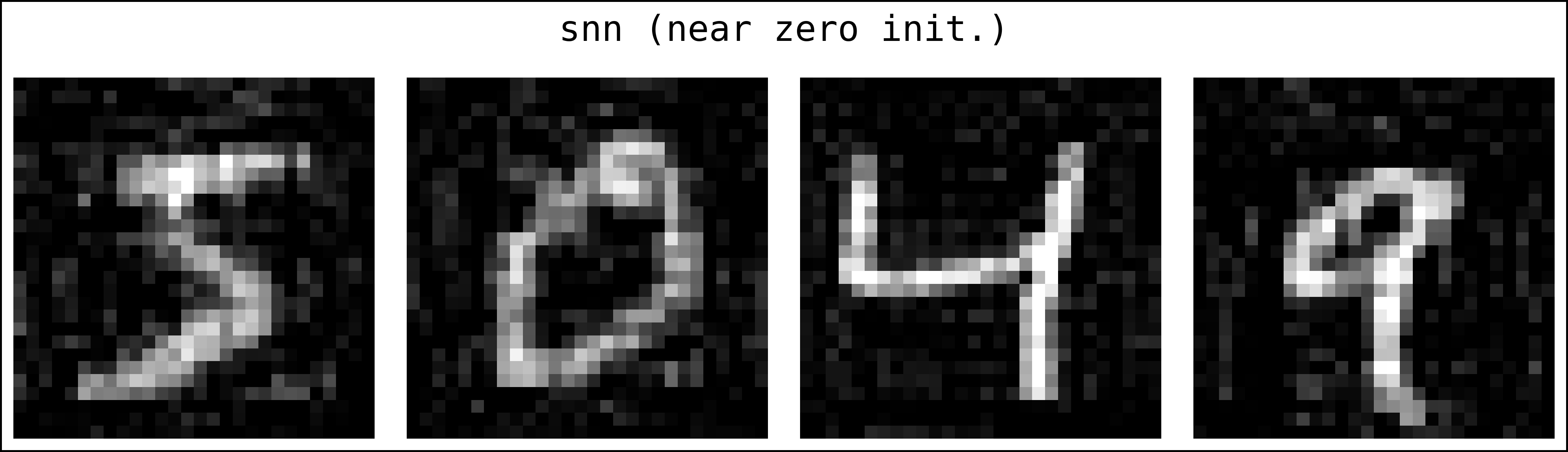}
\caption{Top row: ground-truth \(X^{\star}\), middle row \& bottom row: recovered images by linear regression and our model.}
\label{fig:exp-3}
\end{figure}

\section{Summary of contributions and future directions}
In this work, we investigate the phenomenon of implicit bias via gradient flow in non-linear
spectral parameterizations in matrix sensing. More precisely, we propose a matrix network architecture called Spectral Neural Networks (SNN) where each neuron applies nonlinear activation to the singular values of the input matrix. This novel matrix neural network architecture (i.e., SNN) is particularly well-suited for matrix learning problems. SNNs are characterized by a spectral application of a non-linear activation function to matrix-valued input, rather than an entrywise one. Conceptually, this entails coordinatizing the space of matrices by their singular values and singular vectors, as opposed to by their entries. We believe that this perspective has the potential to gain increasing salience in a wide array of matrix learning scenarios. SNNs are particularly well-suited for theoretical analysis due to the spectral nature of the non-linearities, as opposed to vanilla neural nets, while at the same time provably guaranteeing effectiveness in matrix learning problems. 
Our theoretical contributions include a simple analytical representation of the gradient flow dynamics, accorded by the spectral nature of our network architecture. We demonstrate a rigorous proof of exponentially fast convergence of loss under gradient flow (which is a continuous surrogate
of the discrete gradient descent) to an approximation to the original matrix that is best in a certain class. %
Finally, we demonstrate that, using suitable near zero initialization, the matrix-valued limit of the gradient flow dynamics achieves zero training loss and is approximately a minimizer of the matrix nuclear norm, thereby rigorously establishing the phenomenon of implicit bias in this setting.

Our work raises several exciting possibilities for follow-up and future research. A natural direction is to extend our analysis to the most general setting in which the sensing matrices $A_i$ are non-commuting. 
An investigation of the dynamics in the setting of discrete time gradient descent (as opposed to continuous time gradient flow) is an important question, wherein the optimal choice of the learning rate appears to be an intriguing question, especially in the context of our numerical studies (cf. Fig. \ref{fig:exp-2}). 
Finally, it would be of great interest to develop a general theory of SNNs for applications of neural network-based techniques to matrix learning problems.

\section*{Acknowledgements and AI Statement}
\subsection*{AI Statement}
In this work, we have used standard generative AI tools to support our work.  The authors supplied the research directions and detailed proof ideas, and standard generative AI tools were used for literature search, calculations, the improvement and simplification of parts of the proofs and typesetting. The authors, in particular, proposed the SNN architecture and gave initial proof of the implicit bias (see the earlier version of this paper \cite{chu2024implicit}).
The authors have verified all steps, and take responsibility for the final content of this work.

\subsection*{Acknowledgement}
S.G. was supported in part by the MOE grants R-146-000-250-133, R-146-000-312-114, A-8002014-00-00 and MOE-T2EP20121-0013. S.S.M. was supported by the Prime Minister Early Career Research Grant \seqsplit{ANRF/ECRG/2024/006704/PMS} from the Anusandhan National Research Foundation, Govt.~of India.
\bibliography{reference}

\end{document}